\documentclass[10pt,letterpaper]{article}
\usepackage[margin=1in]{geometry}
\makeatletter
\def\thm@space@setup{%
\thm@preskip=1em \thm@postskip=0pt
}
\makeatother

\usepackage{amsmath, amsthm, amssymb, bbold, mathtools}
\usepackage{footnote}

\usepackage{acronym}

\usepackage{subfig}

\usepackage{caption}

\usepackage{makecell}

\usepackage{import}
\usepackage{cite}
\usepackage{enumerate}
\usepackage{hyperref}
\usepackage[shortlabels]{enumitem}

\usepackage{fancyhdr}
\usepackage{lastpage}

\usepackage{makecell}

\numberwithin{equation}{section}
\newtheorem{assumption}{Assumption}[section]

\newtheorem{proposition}{Proposition}[section]

\newtheorem{definition}{Definition}[section]
\newtheorem{theorem}{Theorem}[section]
\newtheorem{corollary}{Corollary}[section]
\newtheorem{example}{Example}[section]
\newtheorem{remark}{Remark}[section]

\acrodef{ouq}[OUQ]{optimal uncertainty quantification}
\acrodef{dro}[DRO]{distributionally robust optimization}
\acrodef{ldt}[LDT]{large deviation theory}
\acrodef{ldp}[LDP]{large deviation principle}
\acrodef{lln}[LLN]{law of large numbers}
\acrodef{kl}[KL]{Kullback-Leibler}
\acrodef{iid}[{i.i.d.\ \!\!}]{independent identically distributed}
\acrodef{qp}[QP]{quadratic program}
\acrodef{qcqp}[QCQP]{quadratically constrained quadratic program}
\acrodef{vod}[VoD]{value of data}
\acrodef{saa}[SAA]{stochastic average approximation}
\acrodef{fl}[FL]{Fenchel-Legendre}

\usepackage{tikz}
\usepackage{pgfplots}
\usetikzlibrary{plotmarks,arrows,shapes,backgrounds}
\usetikzlibrary{automata}
\usetikzlibrary{matrix,positioning}
\usetikzlibrary{decorations.pathmorphing,calc}
\usepackage{tikz-3dplot}
\usetikzlibrary{external}
\usetikzlibrary{patterns}
\usepackage{framed}
\usepackage{bm}

\usepackage{xcolor,calc}
\newcommand{\noopsort}[1]{}

\newcommand{\drop}[1]{}

\usepackage{cleveref}

\definecolor{darkgreen}{rgb}{0,0.685,0}

\newcommand{\norm}[1]{\left\|#1\right\|}

\newcommand{\D}[2]{\mathsf{D}(#1 \Vert #2 )}

\newcommand{\mc}{\mathcal}
\newcommand{\mb}{\mathbb}
\renewcommand{\emph}{\textbf}
\def\d{\mathrm{d}}

\def\Re{\mathbb{R}}

\def\st{\mathrm{s.t.}}

\DeclareMathOperator{\cl}{cl}
\DeclareMathOperator{\interior}{int}

\title{\bf Robust Generalization despite Distribution Shift via Minimum Discriminating Information}
\author{Tobias Sutter$^1$ \and Andreas Krause$^2$ \and Daniel Kuhn$^3$}
\date{\small{
    $^1$Department of Computer Science, University of Konstanz, {tobias.sutter@uni-konstanz.de}\\%
    $^2$Department of Computer Science, ETH Zurich, {krausea@ethz.ch}\\
    $^3$Risk Analytics and Optimization Chair, Ecole Polytechnique F\'ed\'erale de Lausanne, {daniel.kuhn@epfl.ch}\\ [2ex]}
    \today
}
\providecommand{\keywords}[1]{\textbf{\textit{Keywords---}} #1}

\begin{document}

\maketitle

\begin{abstract}
Training models that perform well under distribution shifts is a central challenge in machine learning. In this paper, we introduce a modeling framework where, in addition to training data, we have partial structural knowledge of the shifted test distribution. We employ the principle of {\em minimum discriminating information} to embed the available prior knowledge, and use {\em distributionally robust optimization} to account for uncertainty due to the limited samples. By leveraging large deviation results, we obtain explicit generalization bounds with respect to the unknown shifted distribution. Lastly, we demonstrate the versatility of our framework by demonstrating it on two rather distinct applications: (1) training classifiers on systematically biased data and (2) off-policy evaluation in Markov Decision Processes.   

\end{abstract}
\acresetall

\section{Introduction} \label{sec:introduction}

Developing machine learning-based systems for real world applications is challenging, particularly because the conditions under which the system was trained are rarely the same as when using the system. Unfortunately, a standard assumption in most machine learning methods is that test and training distribution are the \textit{same} \cite{vapnik1998statistical,ref:Schoelkopf-01, ref:Bishop-06}. This assumption, however, rarely holds in practice, and the performance of many models suffers in light of this issue, often called \textit{dataset shift}~\cite{ref:distribution:shift:book} or equivalently \textit{distribution shift}.
Consider building a model for diagnosing a specific heart disease, and suppose that most participants of the study are middle to high-aged men. Further suppose these participants have a higher risk for the specific disease, and as such do not reflect the general population with respect to age and gender. Consequently, the training data suffers from the so-called \textit{sample selection bias} inducing a \textit{covariate shift} \cite{ref:Shimodaira-00,ref:distribution:shift:book}. Many other reasons lead to distribution shifts, such as non-stationary environments \cite{ref:Masashi-12}, imbalanced data \cite{ref:distribution:shift:book}, domain shifts \cite{ref:Shai-07}, label shifts \cite{zhang2020coping} or observed contextual information \cite{ref:Kallus-20,ref:Bart-Bootstrap:Prescriptive}.
A specific type of distribution shift takes center stage in off-policy evaluation (OPE) problems. Here, one is concerned with the task of estimating the resulting cost of an \textit{evaluation policy} for a sequential decision making problem based on historical data obtained from a different policy known as \textit{behavioral policy} \cite{ref:Sutton1998}. This problem is of critical importance in various applications of reinforcement learning---particularly, when it is impossible or unethical to evaluate the resulting cost of an evaluation policy by running it on the underlying system. \\
Solving a learning problem facing an arbitrary and unknown distribution shift based on training data in general is hopeless. Oftentimes, fortunately, partial knowledge about the distribution shift is available. In the medical example above, we might have prior information how the demographic attributes in our sample differ from the general population.
Given a training distribution and partial knowledge about the shifted test distribution, one might ask what is the ``most natural" distribution shift mapping the training distribution into a test distribution consistent with the available structural information. Here, we address this question, interpreting ``most natural" as maximizing the underlying Shannon entropy. This concept has attracted significant interest in the past in its general form, called \textit{principle of minimum discriminating information} dating back to Kullback \cite{kullback1959information}, which can be seen as a generalization of Jaynes' \textit{maximum entropy principle} \cite{jaynes57_2}.
While these principles are widely used in tasks ranging from economics \cite{ref:Golan-08} to systems biology \cite{ref:Smadback-15} and regularized Markov decision processes \cite{ref:neu2017unified, ref:Geist-19,ref:Peters:MDP-19}, they have not been investigated to model general distribution shifts as we consider in this paper. \\
Irrespective of the underlying distribution shift, the training distribution of any learning problem is rarely known, and one typically just has access to finitely many training samples. 
It is well-known that models can display a poor out-of-sample performance if training data is sparse. These overfitting effects are commonly avoided via regularization \cite{ref:Bishop-06}. A regularization technique that has become popular in machine learning during the last decade and provably avoids overfitting is {\em distributionally robust optimization (DRO)} \cite{ref:DROtutorial-19}.

\textbf{Contributions.} We highlight the following main contributions of this paper:
\begin{itemize}
\item We introduce a {\em new modelling framework} for distribution shifts via the {\em principle of minimum discriminating information}, which encodes prior structural information on the resulting test distribution.
\item Using our framework and the available training samples, we provide {\em generalization bounds} via a DRO program and prove that the introduced DRO model is {\em optimal} in a precise statistical sense.
\item We show that the optimization problems characterizing the distribution shift and the DRO program can be {\em efficiently solved} by exploiting convex duality and recent accelerated first order methods.
\item We demonstrate the {\em versatility} of the proposed {\em Minimum Discriminating based DRO} (MDI-DRO) method on two distinct problem classes: Training classifiers on systematically biased data and the OPE for Markov decision processes. In both problems MDI-DRO outperforms existing approaches.
\end{itemize}

The proofs of all technical results are relegated to Appendix~\ref{sec:appendix}.

\section{Related work}
 For supervised learning problems, there is a rich literature in the context of covariate shift adaptation \cite{ref:Shimodaira-00,ref:Sugiyama-05}. A common approach is to address this distribution shift via importance sampling, more precisely by weighting the training loss with the ratio of the test and training densities and then minimize the so-called importance weighted risk (IWERM), see \cite{ref:Shimodaira-00, ref:Zadrozny-04, ref:Sugiyama-05, JMLR:v8:sugiyama07a}. While this importance weighted empirical risk is an unbiased estimator of the test risk, the method has two major limitations: It tends to produce an estimator with high variance, making the resulting test risk large. Further, the ratio of the training and test densities must be estimated which in general is difficult as the test distribution is unknown. There are modifications of IWERM reducing the resulting variance \cite{ref:Cortes-10,ref:Peters-13,ref:Strehl-10}, for example by exponentially flattening the importance ratios \cite{ref:Shimodaira-00}. For the estimation of the importance weights several methods have been presented, see for example \cite{ref:Yamada-11}. These methods, however crucially rely on having data from both training and test distribution. \cite{ref:Liu-14} and \cite{ref:Chen-16} propose a minimax approach for regression problems under covariate shift. Similar to our approach taken in this paper, they consider a DRO framework, which however, optimizes over so-called moment-based ambiguity sets.
 Distribution shifts play a key role in causal inference. In particular, the connection between causal predictors and distributional robustness under shifts arising from interventions has been widely studied \cite{ref:Rothenhausler-18,ref:Meinhausen-18,ref:Rojas-18,ref:Subbaswamy-19}. Oftentimes, a causal graph is used to represent knowledge about the underlying distribution shift induced by an intervention \cite{ref:Pearl-11, ref:Peters-17}. Distribution shifts have been addressed in a variety of different settings \cite{ref:cui-18}, we refer the reader to the comprehensive textbook \cite{ref:distribution:shift:book} and references therein. \\
 There is a vast literature on OPE methods which we will not attempt to summarize. In a nutshell, OPE methods can be grouped into three classes: a first class of approaches that aims to fit a model from the available data and uses this model then to estimate the performance of the given evaluation policy \cite{ref:Mannor-07, ref:Csaba:OPE-06, ref:Lagoudakis-03}. A second class of methods are based on invoking the idea of importance sampling to model the underlying distribution shift from behavioral to evaluation policy~\cite{ref:Precup-00,ref:Hirano-03,ref:Joachims-15}. The third, more recent, class of methods combines the first two classes~\cite{ref:Dudik-14,ref:Jiang-16,pmlr-v48-thomasa16,ref:Kallus-20-RL}.\\
Key reasons for the popularity of DRO in machine learning are the ability of DRO models to regularize learning problems \cite{ref:DROtutorial-19,ref:soroosh-15,ref:soroosh:JMLR-19} and the fact that the underlying optimization problems can often be exactly reformulated as finite convex programs solvable in polynomial time \cite{ben2009robust, bertsimas2004price}. Such reformulations hold for a variety of ambiguity sets such as: regions defined by moments \cite{delage2010distributionally,goh2010distributionally, wiesemann2014distributionally, ref:Kallus-18}, $\phi$-divergences \cite{ref:BenTal-13,NIPS2016_4588e674,Li:ICML-21}, Wasserstein ambiguity sets \cite{Wass17Monh, ref:DROtutorial-19}, or maximum mean discrepancy ambiguity sets \cite{ref:Jegelka:MMD-19,ref:Kirschner-20}. DRO naturally seems a convenient tool when analyzing ``small" distribution shifts as it seeks models that perform well ``sufficiently close" to the training sample. However, modelling a general distribution shift via DRO seems difficult, and recent interest has focused on special cases such as adversarial example shifts \cite{duchi2020learning} or label shifts \cite{zhang2020coping}. To the best of our knowledge, the idea of combining DRO with the principle of minimum discriminating information is new.

\section{Problem statement and motivating examples} \label{sec:problem:statement}
We study learning problems of the form
\begin{equation}\label{eq:SP:1} 
\min_{\theta\in\Theta} R(\theta,\mathbb{P}^\star),
\end{equation}
where $R(\theta,\mathbb{P}^\star)= \mathbb{E}_{\mathbb{P}^\star}[L(\theta,\xi)]$ denotes the risk of an uncertain real-valued loss function~$L(\theta,\xi)$ that depends on a parameter $\theta\in\Theta\subset \Re^n$ to be estimated as well as a random vector $\xi\in\Xi\subset \Re^{m}$ governed by the probability distribution~$\mathbb P^\star$. In order to avoid technicalities, we assume from now on that~$\Theta$ and~$\Xi$ are compact and $L$ is continuous. In statistical learning, it is usually assumed that~$\mathbb P^\star$ is unknown but that we have access to independent samples from~$\mathbb P^\star$. This paper departs from this standard scenario by assuming that there is a distribution shift. We first state our formal assumption about the shift and provide concrete examples below. Specifically, we assume to have access to samples from a distribution~$\mathbb P\neq\mathbb P^\star$ and that~$\mathbb P^\star$ is only known to belong to the distribution family
\begin{equation}\label{eq:Pi}
\textstyle{\Pi = \left\{ \mathbb{Q}\in\mathcal{P}(\Xi) \ : \ \mathbb E_{\mathbb{Q}} \left[\psi(\xi)\right] \in E \right\}}
\end{equation}
encoded by a measurable feature map $\psi:\Xi\to\Re^d$ and a compact convex set~$E\subset\Re^d$. In view of the principle of minimum discriminating information, we identify~$\mathbb P^\star$ with the I-projection of~$\mathbb P$ onto~$\Pi$.

\begin{definition}[Information projection] \label{def:Iprojection}
The I-projection of $\mathbb{P}\in\mathcal{P}(\Xi)$ onto $\Pi$ is defined as
\begin{equation} \label{eq:def:Iprojection}
\mathbb P^f = f(\mathbb{P})=\arg\min_{\mathbb{Q}\in\Pi} \D{\mathbb{Q}}{\mathbb{P}},
\end{equation}
where~$\D{\mathbb{Q}}{\mathbb{P}}$ denotes the relative entropy of~$\mathbb Q$ with respect to~$\mathbb P$.
\end{definition}
One can show that the I-projection exists whenever~$\Pi$ is closed with respect to the topology induced by the total variation distance \cite[Theorem~2.1]{csiszar1975}. As~$E$ is closed, this is the case whenever~$\psi$ is bounded.
Note that~$f(\mathbb P)=\mathbb P$ if~$\mathbb P\in\Pi$. In the remainder, we assume that~$\mathbb P\not\in\Pi$ and that $\mathbb P$ is only indirectly observable through independent training samples~$\widehat \xi_1,\hdots, \widehat \xi_N$ drawn from~$\mathbb P$.

\begin{example}[Logistic regression] \label{ex:LR:part:1}
\looseness -1 Assume that~$\xi=(x,y)$, where $x\in\Re^{m-1}$ is a feature vector of patient data (e.g., a patient's age, sex, chest pain type, blood pressure, etc.), and $y\in\{-1,1\}$ a label indicating the occurrence of a heart disease. 
Logistic regression models the conditional distribution of~$y$ given~$x$ by a logistic function~$\text{\rm Prob}(y|x)=[1+\exp(-y\cdot \theta^\top x)]^{-1}$ parametrized by $\theta\in \Re^{m-1}$. The maximum likelihood estimator for~$\theta$ is found by minimizing the empirical average of the logistic loss function $L(\theta,\xi)=\log(1+\exp(-y\cdot \theta^\top x))$ on the training samples. If the samples pertain to a patient cohort, where elderly males are overrepresented with respect to the general population, then they are drawn from a training distribution~$\mathbb P$ that differs from the test distribution~$\mathbb P^\star$. Even if sampling from~$\mathbb P^\star$ is impossible, we may know that the expected age of a random individual in the population falls between 40 and 45 years. This information can be modeled as $\mathbb E_{\mathbb P^\star} \left[\psi(\xi)\right] \in E$, where~$E=[\ell, u]$, $\ell=40$, $u=45$ and~$\psi(\xi)$ projects~$\xi$ to its `age'-component. Other available prior information can be encoded similarly. Inspired by the principle of minimum discriminating information, we then minimize the expected log-loss under the I-projection~$\mathbb P^f$ of the data-generating distribution~$\mathbb P$ onto the set~$\Pi$ defined in~\eqref{eq:Pi}.
\end{example}

\begin{example}[Production planning] \label{ex:NW:part1}
Assume that~$\theta\in\mathbb R$ and~$\xi\in\mathbb R$ denote the production quantity and the demand of a perishable good, respectively, and that the loss function $L(\theta,\xi)$ represents the sum of the production cost and a penalty for unsatisfied demand. To find the optimal production quantity, one could minimize the average loss in view of training samples drawn from the historical demand distribution~$\mathbb P$. However, a disruptive event such as 
the beginning of a recession might signal that demand will decline by at least~$\eta\%$. The future demand distribution~$\mathbb P^\star$ thus differs from~$\mathbb P$ and belongs to a set~$\Pi$ of the form~\eqref{eq:Pi} defined through $\psi(\xi) = \xi$ and $E=[0,(1-\eta) \mu]$, where~$\mu$ denotes the historical average demand. By the principle of minimum discriminating information it then makes again sense to minimize the expected loss under the I-projection~$\mathbb P^f$ of~$\mathbb P$ onto~$\Pi$.
\end{example}

Loosely speaking, the principle of minimum discriminating information identifies the I-projection~$\mathbb{P}^f$ of~$\mathbb P$ as the least prejudiced and thus most natural model for~$\mathbb{P}^\star$ in view of the information that~$\mathbb{P}^\star\in\Pi$. The principle of minimum discriminating information is formally justified by the conditional limit theorem \cite{ref:Csiszar-84}, which we paraphrase below using our notation.
\begin{proposition}[Conditional limit theorem] \label{prop:conditional:limit:thm}
If the interior of the compact convex set~$E$ overlaps with the support of the pushforward measure $\mathbb{P}\circ{\psi^{-1}}$, the I-projection $\mathbb{P}^f = f(\mathbb{P})$ exists and the moment-generating function $\mathbb E_{\mathbb P^f}[ e^{t L(\theta,\xi)}]$ is finite for all $t$ in a neighborhood of~$0$, then we have
\begin{equation*}
\lim_{N\to\infty} \mathbb{E}_{\mathbb{P}^N}[L(\theta,\xi_1)|\textstyle{\frac{1}{N}\sum_{i=1}^N\psi(\xi_i)\in E]} = \mathbb{E}_{\mathbb{P}^f}[L(\theta,\xi)]\quad \forall \theta\in\Theta.
\end{equation*}
\end{proposition}

In the context of Examples~\ref{ex:LR:part:1} and~\ref{ex:NW:part1}, the conditional limit theorem provides an intuitive justification for modeling distribution shifts via I-projections. More generally, the following proposition suggests that {\em any} distribution shift can be explained as an I-projection onto a suitably chosen set~$\Pi$.

\begin{proposition}[Every distribution is an I-projection] \label{prop:every:measure:I:projection}
If $\mathbb{P},\mathbb{Q}\in\mathcal{P}(\Xi)$ are such that $\mathbb{Q}$ is absolutely continuous with respect to $\mathbb{P}$ and if~$\Pi$ is a set of the form~\eqref{eq:Pi} defined through $\psi(\xi) = \log \frac{\d \mathbb{Q}}{\d \mathbb{P} }(\xi)$ and $E=\{\D{\mathbb{Q}}{\mathbb{P}}\}$, then
$\mathbb{Q}=f(\mathbb{P})$.
\end{proposition}

\looseness -1 The modelling of arbitrary distribution shifts via the I-projection according to Proposition~\ref{prop:every:measure:I:projection} has an interesting application in the off-policy evaluation problem for Markov decision processes (MDPs).

\begin{example}[Off-policy evaluation] \label{ex:OPE:part:1}
Consider an MDP $(\mathcal{S},\mathcal{A},Q,c, s_0)$ with finite state and action spaces $\mathcal{S}$ and~$\mathcal{A}$, respectively, transition kernel $Q:\mathcal S\times\mathcal A\rightarrow\mathbb R$, cost-per-stage function~$c:\mathcal{S}\times\mathcal{A}\to\Re$ and initial state~$s_0$. A stationary Markov policy $\pi$ is a stochastic kernel that maps states to probability distributions over~$\mathcal A$. We use $\pi(a|s)$ to denote the probability of selecting action~$a$ in state~$s$ under policy~$\pi$. The long-run average cost generated by~$\pi$ can be expressed as
\begin{equation*}
\textstyle{V_\pi= 
\lim_{T\to\infty}\frac{1}{T}\sum_{t=0}^{T-1} \mathbb{E}^\pi_{s_0}[c(s_t,a_t)]}.
\end{equation*}
Each policy induces an occupation measure~$\mu_\pi$ on~$\mathcal S\times \mathcal A$ defined through the state-action frequencies
\begin{equation*}
\textstyle{\mu_\pi(x,a) = \lim_{T\to\infty}\frac{1}{T} \sum_{t=0}^{T-1} \mathbb{P}^\pi_{s_0}[(s_t,a_t) = (s,a)] \quad \forall s\in \mathcal{S}, ~a\in \mathcal{A}},
\end{equation*}
see \cite[Chapter~6]{ref:Hernandez-96}. One can additionally show that $\mu_\pi$ belongs to the polytope 
\begin{equation*}
\textstyle{\mathcal{M}=\left\{\mu\in\Delta_{\mathcal{S}\times \mathcal{A}} : \sum_{a'\in\mathcal A} \mu(s',a') - \sum_{s\in \mathcal{S}} \sum_{a \in \mathcal{A}} Q(s'|s,a)\mu(s,a)=0~ \forall s'\in\mathcal{S} \right\}},
\end{equation*}
where $\Delta_{\mathcal{S}\times \mathcal{A}}$ represents the simplex of all probability mass functions over $\mathcal{S}\times \mathcal{A}$. Conversely, each occupation measure $\mu\in\mathcal{M}$ induces a policy~$\pi_\mu$ defined through $\pi_\mu(a|s) = \mu(s,a)/\sum_{a'\in\mathcal{A}} \mu(s,a')$ for all $s\in\mathcal{S}$ and $a\in\mathcal{A}$. 
Assuming that all parameters of the MDP except for the cost~$c$ are known, the off-policy evaluation problem asks for an estimate of the long-run average cost~$V_{\pi_\mathsf{e}}$ of an evaluation policy~$\pi_\mathsf{e}$ based on a trajectory of states, actions and costs generated by a behavioral policy~$\pi_\mathsf{b}$. This task can be interpreted as a degenerate learning problem without a parameter~$\theta$ to optimize if we define~$\xi=c(s,a)$ and set~$L(\theta,\xi) = \xi$. Here, a distribution shift emerges because we must evaluate the expectation of~$\xi$ under~$\mathbb Q=\mu_{\mathsf e}\circ c^{-1}$ given training samples from~$\mathbb P=\mu_{\mathsf b} \circ c^{-1}$, where~$\mu_{\mathsf b}$ and~$\mu_{\mathsf e}$ represent the occupation measures corresponding to~$\pi_{\mathsf b}$ and~$\pi_{\mathsf e}$, respectively. Note that~$\mathbb P$ and~$\mathbb Q$ are unknown because~$c$ is unknown. Moreover, as the policy~$\pi_\mathsf{e}$ generates different state-action trajectories than~$\pi_\mathsf{b}$, the costs generated under~$\pi_\mathsf{e}$ cannot be inferred from the costs generated under~$\pi_\mathsf{b}$ even though~$\pi_\mathsf{b}$ and~$\pi_\mathsf{e}$ are known. Note also that~$\mathbb Q$ coincides with the I-projection~$\mathbb P^f$ of~$\mathbb P$ onto the set~$\Pi$ defined in Proposition~\ref{prop:every:measure:I:projection}. The corresponding feature map $\psi$ as well as the set $E$ can be computed without knowledge of~$c$ provided that $c$ is invertible. Indeed, in this case we have
\begin{equation*}
  \textstyle{  \psi(\xi_i) = \log \frac{\d\mu_\mathsf{e}\circ c^{-1}}{\d\mu_\mathsf{b}\circ c^{-1}}(\xi_i) = \log \frac{\mu_\mathsf{e}(s_i,a_i)}{\mu_\mathsf{b}(s_i,a_i)}}\quad\text{and}\quad E=\left\{\D{\mu_\mathsf{e}\circ c^{-1}}{\mu_\mathsf{b}\circ c^{-1}}\right\}=\left\{\D{\mu_\mathsf{e}}{\mu_\mathsf{b}}\right\}
\end{equation*}

for any $s_i\in\mathcal{S}$,  $a_i\in\mathcal{A}$ and  
$\xi_i=c(s_i,a_i)$. Note that as~$\mathcal S$ and~$\mathcal A$ are finite, $c$ is generically invertible, that is, $c$ can always be rendered invertible by an arbitrarily small perturbation. 
In summary, we may conclude that the off-policy evaluation problem reduces to an instance of~\eqref{eq:SP:1}. 
\end{example}

Given~$N$ training samples $\widehat \xi_1,\ldots, \widehat \xi_N$, we henceforth use~$\widehat{\mathbb{P}}_N=\frac{1}{N}\sum_{i=1}^N \delta_{\widehat \xi_i}$ and~$\widehat{\mathbb{P}}^f_N$ to denote the empirical distribution on and its I-projection onto~$\Pi$, respectively. As the true data-generating distribution~$\mathbb P$ and its I-projection~$\mathbb P^f$ are unknown, it makes sense to replace them by their empirical counterparts. However, the resulting empirical risk minimization problem is susceptible to overfitting if the number of training samples is small relative to the feature dimension. In order to combat overfitting, we propose to solve the DRO problem
\begin{equation} \label{eq:def:DRO:general}
    J^\star_N=\min_{\theta\in \Theta}~ R^\star(\theta,\widehat{\mathbb{P}}^f_N), 
\end{equation}
which minimizes the worst-case risk over all distributions close to~$\widehat{\mathbb{P}}^f_N$. Here, $R^\star$ is defined through
\begin{equation}\label{eq:def:DRO:predictor}
    \textstyle R^\star(\theta,\mathbb P') = \sup_{{\mathbb{Q}}\in\Pi}\left\{ R(\theta,\mathbb Q): \D{\mathbb{P}'}{{\mathbb{Q}}}\leq r\right\}
\end{equation}
and thus evaluates the worst-case risk of a given parameter~$\theta\in\Theta$ in view of all distributions~$\mathbb Q$ that have a relative entropy distance of at most~$r$ from a given nominal distribution~$\mathbb P'\in\Pi$.
In the remainder we use~$J^\star_N$ and~$\theta^\star_N$ to denote the minimum and a minimizer of problem~\eqref{eq:def:DRO:general}, respectively.  

\textbf{Main results.}
The main theoretical results of this paper can be summarized as follows. \vspace{-3mm}
\begin{enumerate}
\item \label{item:des:prop:a} \textit{Out-of-sample guarantee.} We show that the optimal value of the DRO problem~\eqref{eq:def:DRO:general} provides an upper confidence bound on the risk of its optimal solution~$\theta^\star_N$. Specifically, we prove that
\begin{align}\label{eq:oos:guarantees:motivation}
\mathbb{P}^N\left( R(\theta^\star_N,\mathbb{P}^f) > J^\star_N \right)\leq e^{-rN+o(N)},
\end{align}
where~$\mathbb P^f=f(\mathbb P)$ is the I-projection of~$\mathbb P$. If~$\Xi$ is finite, then~\eqref{eq:oos:guarantees:motivation} can be strengthened to a finite sample bound that holds for every~$N$ if the right hand side is replaced with $e^{-rN}(N+1)^{|\Xi|}$.
\item \label{item:des:prop:b} \textit{Statistical efficiency.} In a sense to be made precise below, the DRO problem~\eqref{eq:def:DRO:general} provides the least conservative approximation for~\eqref{eq:SP:1} whose solution satisfies the out-of-sample guarantee~\eqref{eq:oos:guarantees:motivation}.
\item \label{item:des:prop:c} \textit{Computational tractability.} 
We prove that the I-projection $\widehat{\mathbb{P}}^f_N$ can be computed via a regularized fast gradient method whenever one can efficiently project onto~$E$. Given $\widehat{\mathbb{P}}^f_N$, we then show that~$\theta^\star_N$ can be found by solving a tractable convex program whenever $\Theta$ is a convex and conic representable set, while $L(\theta,\xi)$ is a convex and conic representable function of~$\theta$ for any fixed~$\xi$.
\end{enumerate}


\section{Statistical guarantees} \label{sec:stat:guarantees}

Throughout this section, we equip~$\mc P(\Xi)$ with the topology of weak convergence. As $L(\theta,\xi)$ is continuous on~$\Theta\times\Xi$ and~$\Xi$ is compact, this implies that the risk~$R(\theta, \mathbb Q)$ is continuous on~$\Theta\times\mathcal P(\Xi)$. The DRO problem~\eqref{eq:def:DRO:general} is constructed from the I-projection of the empirical distribution, which, in turn, is constructed from the given training samples. Thus, $\theta_N^\star$ constitutes a data-driven decision. Other data-driven decisions can be obtained by solving surrogate optimization problems of the form
\begin{equation} \label{eq:def:surrogate-problem}
    \widehat J_N=\min_{\theta\in \Theta}~ \widehat R(\theta,\widehat{\mathbb P}^f_N),
\end{equation}
where $\widehat R:\Theta\times\Pi\to\Re$ is a continuous function that uses the empirical I-projection~$\widehat {\mathbb P}^f_N$ to predict the true risk~$R(\theta,\mathbb P^f)$ of~$\theta$ under the true I-projection~$\mathbb P^f$. From now on we thus refer to~$\widehat R$ as a predictor, and we use~$\widehat J_N$ and~$\widehat \theta_N$ to denote the minimum and a minimizer of problem~\eqref{eq:def:surrogate-problem}, respectively. We call a predictor~$\widehat R$ {\em admissible} if~$\widehat J_N$ provides an upper confidence bound on the risk of~$\widehat \theta_N$ in the sense that
\begin{align}\label{eq:oos:guarantees:surrogate}
\limsup\limits_{N\to\infty} \frac{1}{N} \log \mathbb{P}^N\left( R(\widehat\theta_N,\mathbb{P}^f) > \widehat J_N \right)\leq -r \quad\forall\mathbb P\in\mathcal P(\Xi)
\end{align}
for some prescribed~$r>0$. The inequality~\eqref{eq:oos:guarantees:surrogate} requires the true risk of the minimizer~$\widehat \theta_N$ to exceed the optimal value~$\widehat J_N$ of the surrogate optimization problem~\eqref{eq:def:surrogate-problem} with a probability that decays exponentially at rate~$r$ as the number~$N$ of training samples tends to infinity. The following theorem asserts that the~DRO predictor~$R^\star$ defined in~\eqref{eq:def:DRO:predictor}, which evaluates the worst-case risk of any given~$\theta$ across a relative entropy ball of radius~$r$, almost satisfies~\eqref{eq:oos:guarantees:surrogate} and is thus essentially admissible.

\begin{theorem}[Out-of-sample guarantee] \label{thm:admissibility}
If $r>0$, $0\in\text{\rm int}(E)$ and for every~$z\in\mathbb R^d$ there exists an uncertainty realization~$\xi\in\Xi$ such that~$z^\top \psi(\xi)>0$, then the DRO predictor $R^\star$ defined in~\eqref{eq:def:DRO:predictor} is continuous on~$\Theta\times\Pi$. In addition, $\widehat R=R^\star+\varepsilon$ is an admissible data-driven predictor for every~$\varepsilon>0$.
\end{theorem}

Theorem~\ref{thm:admissibility} implies that, for any fixed~$\varepsilon>0$, the DRO predictor $R^\star$ provides an upper confidence bound~$J^\star_N+\varepsilon$ on the true risk~$R(\theta^\star_N,\mathbb P^f)$ of the data-driven decision~$\theta^\star_N$ that becomes  increasingly reliable as~$N$ grows. Of course, the reliability of {\em any} upper confidence bound trivially improves if it is increased. Finding {\em some} upper confidence bound is thus easy. The next theorem shows that the DRO predictor actually provides the {\em best possible} (asymptotically smallest) upper confidence bound.

\begin{theorem}[Statistical efficiency] \label{thm:statistical:efficiency}
Assume that all conditions of Theorem~\ref{thm:admissibility} hold. If $J^\star_N$ and $R^\star$ are defined as in~\eqref{eq:def:DRO:general} and~\eqref{eq:def:DRO:predictor}, while $\widehat J_N$ is defined as in~\eqref{eq:def:surrogate-problem} for any admissible data-driven predictor~$\widehat R$, then we have $\lim_{N\rightarrow\infty} J^\star_N\leq \lim_{N\rightarrow\infty} \widehat J_N$ $\mathbb P^\infty$-almost surely irrespective of~$\mathbb P\in\mathcal P(\Xi)$. 
\end{theorem}

One readily verifies that the limits in Theorem~\ref{thm:statistical:efficiency} exist. 
Indeed, if $\widehat R$ is an arbitrary data-driven predictor, then the optimal value~$\widehat J_N$ of the corresponding surrogate optimization problem converges $\mathbb P$-almost surely to~$\min_{\theta\in\Theta} \widehat R(\theta, \mathbb P^f)$ as $N$ tends infinity provided that the training samples are drawn independently from~$\mathbb P$. This is a direct consequence of the following three observations.  First, the optimal value function $\min_{\theta\in\Theta}\widehat R(\theta,\mathbb{P}^f)$ is continuous in~$\mathbb P^f\in\Pi$ thanks to Berge's maximum theorem \cite[pp.~115--116]{berge1997topological}, which applies because $\widehat R$ is continuous and~$\Theta$ is compact. Second, the I-projection $\mathbb P^f=f(\mathbb P)$ is continuous in~$\mathbb P\in\mathcal P(\Xi)$ thanks to \cite[Theorem~9.17]{sundaram_1996}, which applies because the relative entropy is strictly convex in its first argument \cite[Lemma~6.2.12]{dembo2009large}. Third, the strong law of large numbers implies that the empirical distribution $\widehat{\mathbb{P}}_N$ converges weakly to the data-generating distribution~$\mathbb{P}$ as the sample size~$N$ grows. Therefore, we have
\[
    \lim_{N\rightarrow \infty} \widehat J_N = \lim_{N\rightarrow \infty} \min_{\theta\in\Theta} \widehat R\left(\theta,f(\widehat{\mathbb P}_N)\right) = \min_{\theta\in\Theta} \widehat R \left(\theta,f \left(\lim_{N\rightarrow \infty} \widehat {\mathbb P}_N\right)\right)= \min_{\theta\in\Theta} \widehat R(\theta,\mathbb P^f)\quad\mathbb P\text{-a.s.}
\]

In summary, Theorems~\ref{thm:admissibility} and~\ref{thm:statistical:efficiency} assert that the DRO predictor~$R^\star$ is (essentially) admissible and that it is the least conservative of all admissible data-driven predictors, respectively. Put differently, the DRO predictor makes the most efficient use of the available data among all data-driven predictors that offer the same out-of-sample guarantee~\eqref{eq:oos:guarantees:surrogate}. 
In the special case when~$\Xi$ is finite, the asymptotic out-of-sample guarantee~\eqref{eq:oos:guarantees:surrogate} can be strengthened to a finite sample guarantee that holds for every~$N\in\mathbb N$.

\begin{corollary}[Finite sample guarantee]
\label{cor:finite-sample-guarantee}
If $R^\star$ is defined as in \eqref{eq:def:DRO:predictor}, then 
\begin{align}\label{eq:oos:guarantees:finite-sample}
\frac{1}{N} \log \mathbb{P}^N\left( R^\star(\theta^\star_N,\mathbb{P}^f) > J^\star_N \right)\leq \frac{\log(N+1)}{N}|\Xi|-r\quad\forall N\in\mathbb N.
\end{align}
\end{corollary} 

We now temporarily use~$R^\star_r$ to denote the DRO predictor defined in~\eqref{eq:def:DRO:predictor}, which makes its dependence on~$r$ explicit. Note that if~$r>0$ is kept constant, then~$R_r^\star(\theta,\widehat{\mathbb{P}}_N^f)$ is neither an unbiased nor a consistent estimator for~$R(\theta,\mathbb{P}^f)$. Consistency can be enforced, however, by shrinking~$r$ as~$N$ grows.

\begin{theorem}[Asymptotic consistency]\label{thm:asymptotic:consistency}
Let the assumptions of Proposition~\ref{prop:conditional:limit:thm} hold and~$\{r_N\}_{N\in\mathbb{N}}$ be a sequence of non-negative reals with $\lim_{N\to \infty}r_N =0$. If the loss function $L(\theta,\xi)$ is Lipschitz continuous in~$\xi$ with Lipschitz constant~$\Lambda>0$ uniformly across all~$\theta\in\Theta$, then we have 
\begin{subequations}
\begin{align}
&\label{eq:thm:consistency-assertion1}\lim_{N\to\infty}  R_{r_N}^\star(\theta,\widehat{\mathbb{P}}^f_N) = R(\theta,\mathbb{P}^f)\quad \mathbb{P}^\infty\text{-a.s.} ~ \forall \theta\in\Theta,\\
&\label{eq:thm:consistency-assertion2}\lim_{N\to\infty}  \min_{\theta\in\Theta} R_{r_N}^\star(\theta,\widehat{\mathbb{P}}^f_N) = \min_{\theta\in \Theta} R(\theta,\mathbb{P}^f)\quad \mathbb{P}^\infty\text{-a.s.} 
\end{align}
\end{subequations}
\end{theorem}

\begin{remark}[Choice of radius]
Theorem~\ref{thm:statistical:efficiency} shows that the ambiguity set used in our paper displays a strong Pareto-optimality property, i.e., it leads to the least conservative predictor, uniformly across all estimator realizations, for which the out-of-sample disappointment probability is guaranteed to decay exponentially at rate $r$. Therefore, the radius $r$ has a direct operational interpretation that captures the risk tolerance of the decision maker---it is chosen subjectively. Since the statistical guarantees of Theorem~\ref{thm:admissibility} are  asymptotic, selecting the radius $r$ when we only have access to finitely many samples is challenging, and in practice $r$ is usually selected via cross validation.
\end{remark}
We now exemplify our DRO approach and its statistical guarantees in the context of the off-policy evaluation problem introduced in Section~\ref{sec:problem:statement}.
\begin{example}[Off-policy evaluation] \label{ex:OPE:part:2}
Consider again the OPE problem introduced in Example~\ref{ex:OPE:part:1}. We now aim to construct an estimator for the performance of the evaluation policy $V_{\pi_\mathsf{e}}=\mathbb{E}_{f(\mathbb{P})}[\xi]$ based on the available behavioral policy and its empirical cost. As described in Example~\ref{ex:OPE:part:1}, we choose $\Pi$ such that $\mu_\mathsf{e}\circ c^{-1} = f(\mathbb{P})$, where $\mathbb{P}=\mu_\mathsf{b}\circ c^{-1}\in \mathcal{P}(\Xi)$. Given the behavioral data $(\widehat s_i,\widehat a_i)\sim \mu_\mathsf{b}$ for $i=1,\hdots N$, we then construct the empirical distribution $\widehat{\mathbb{P}}_N = \frac{1}{N}\sum_{i=1}^{N} \delta_{c(\widehat s_i,\widehat a_i)}$. Our statistical results require the samples $(\widehat s_i,\widehat a_i)$ to be i.i.d., which can be enforced approximately by discarding a sufficient number of intermediate samples, for example. We emphasize, however, that the proposed large deviation framework readily generalizes to situations in which there is a single trajectory of correlated data \cite{Li:ICML-21, ref:Sutter-19}. Details are omitted for brevity. 
The value function $V_{\pi_\mathsf{e}}$ under the evaluation policy can now be approximated by $J^\star_N=R^\star(\widehat{\mathbb{P}}_N^f)$, where $R^\star$ denotes the DRO predictor \eqref{eq:def:DRO:predictor}. As $\Xi$ is finite, Corollary~\ref{cor:finite-sample-guarantee} provides the generalization bound
\begin{equation}\label{eq:OPE:generalization}
    \mathbb{P}^N \left( V_{\pi_\mathsf{e}} \leq J^\star_N \right)\geq 1- (N+1)^{|\mathcal{S}|+|\mathcal{A}|} e^{-r N} \quad \forall \mathbb{P}\in\mathcal{P}(\Xi),
\end{equation}
which holds for all $N\in\mathbb{N}$.
\end{example}

\section{Efficient computation} \label{sec:computation}

\looseness -1 
We now outline an efficient procedure to solve the DRO problem~\eqref{eq:def:DRO:general}. This procedure consists of two steps. First, we propose an algorithm to compute the I-projection~$\widehat{\mathbb{P}}^f_N = f(\widehat{\mathbb{P}}_N)$ of the empirical distribution~$\widehat{\mathbb{P}}_N$ corresponding to the training samples~$\widehat \xi_1,\hdots,\widehat \xi_N$.
Given $\widehat{\mathbb{P}}^f_N$, we then show how to compute the worst-case risk~$R^\star(\theta,\widehat{\mathbb{P}}_N^f)$ and a corresponding optimizer~$\theta^\star_N$ over the search space~$\Theta$.

\textbf{Computation of the I-projection.}
Computing the I-projection of the empirical distribution~$\widehat{\mathbb{P}}_N$ is a non-trivial task because it requires solving the infinite-dimensional optimization problem~\eqref{eq:def:Iprojection}. Generally, one would expect that the difficulty of evaluating~$f(\widehat{\mathbb{P}}_N)$ depends on the structure of the set $\Pi$, which is encoded by~$\psi$ and~$E$; see~\eqref{eq:Pi}. 
Thanks to the discrete nature of the empirical distribution~$\widehat{\mathbb{P}}_N$, however, we can leverage recent advances in convex optimization together with an algorithm proposed in~\cite{ref:sutter-JMLR-19} to show that~$f(\widehat{\mathbb{P}}_N)$ can be evaluated efficiently for a large class of sets~$\Pi$.

In the following we let~$\eta=(\eta_1,\eta_2)$ be a smoothing parameter with $\eta_1,\eta_2>0$, and we let~$L_\eta>0$ be a learning rate that may depend on $\eta$. In addition, we denote by~$z\in\Re^d$ the vector of dual variables of the constraint~$\mathbb E_{\mathbb Q}[\psi(\xi)]\in E$ in problem~\eqref{eq:def:Iprojection}, and we define~$G_\eta:\Re^d\to\Re^d$ with
\begin{equation} \label{eq:def:gradient:FGA}
\textstyle{G_\eta(z) = -\pi_E(\eta_1^{-1}z) - \eta_2 z + \frac{\sum_{i=1}^N\psi(\widehat\xi_i)\exp\left(-\sum_{j=1}^dz_j \,\psi_j(\widehat\xi_i)\right)}{\sum_{i=1}^N \exp\left(-\sum_{j=1}^dz_j \,\psi_j(\widehat\xi_i)\right)}}
\end{equation}
as a smoothed gradient of the dual objective, where $\pi_E$ denotes the projection operator onto~$E$ defined through $\pi_E\left(z\right)=\arg\min_{x\in E} \| x - z\|_2^2$. The corresponding smoothed dual of the I-projection problem~\eqref{eq:def:Iprojection} can then be solved with the fast gradient method described in Algorithm~\hyperlink{algo:1}{1}. The complexity of evaluating~$G_\eta$, and thus the per-iteration complexity of Algorithm~\hyperlink{algo:1}{1}, is determined by the projection operator onto $E$. For
simple sets (e.g., 2-norm balls or hybercubes) the solution is available in closed form,
and for many other sets (e.g., simplices or 1-norm balls) it can be computed cheaply, see \cite[Section 5.4]{RichterPhD2012} for a comprehensive survey.

 \begin{table}[!htb]
\centering 
\begin{tabular}{m{45em}} 
  \Xhline{1.2pt}  
\hspace{12.2mm}{\vspace{-3.5mm}\bf{\hypertarget{algo:1}{Algorithm 1: } }} Optimal scheme for smooth $\&$ strongly convex optimization \cite{ref:Nesterov:book:14} \hspace{24.2mm} \\ \vspace{-10mm} \\ \hline \vspace{-12mm}
\end{tabular} 
\vspace{-4mm}
 \begin{flushleft}
  {\hspace{3mm}Choose $w_0=y_{0} \in \Re^{d}$ and $\eta\in\Re_{++}^2$}
 \end{flushleft}
 \vspace{-0mm}
 \begin{flushleft}
  {\hspace{3mm}\bf{For $k\geq 0$ do}}
 \end{flushleft}
 \vspace{-5.5mm}
  \begin{tabular}{l l}
{\bf Step 1: } & Set $y_{k+1}=w_{k}+\frac{1}{L_\eta}G_{\eta}(w_{k})$ \\
{\bf Step 2: } & Compute $w_{k+1}=y_{k+1} + \frac{\sqrt{L_\eta}-\sqrt{\eta_{2}}}{\sqrt{L_\eta}+\sqrt{\eta_{2}}}(y_{k+1}-y_{k})$\\
  \end{tabular}
   \begin{flushleft}
  \vspace{-12mm}
 \end{flushleft}  
\begin{tabular}{c}
\hspace{12.2mm} \phantom{ {\bf{Algorithm:}} Optimal Scheme for Smooth $\&$ Strongly Convex Optimization}\hspace{26.5mm} \\ \vspace{0.75mm} 
\end{tabular}
\begin{tabular}{m{45em}}
  \Xhline{1.2pt}   \vspace{-1mm}\\ 
\end{tabular}
\end{table}

 Any output $z_k$ of Algorithm~\hyperlink{algo:1}{1} after $k$ iterations can be used to construct a candidate solution
\begin{equation} \label{eq:estimates:primal:dual}
\textstyle{
\widehat{\mb Q}_{k} =\frac{\sum_{j=1}^N\exp\left(-\sum_{i=1}^d(z_{k})_i\psi_i \,(\widehat\xi_j)\right) \delta_{\widehat\xi_j}}{\sum_{j=1}^N\exp\left(-\sum_{i=1}^d(z_{k})_i \,\psi_i(\widehat\xi_j)\right)}}
\end{equation}
for problem~\eqref{eq:def:Iprojection} that approximates the I-projection~$\widehat{\mathbb{P}}^f_N$. The convergence guarantees for  Algorithm~\hyperlink{algo:1}{1} and, in particular, the approximation quality of~\eqref{eq:estimates:primal:dual} with respect to~$\widehat{\mathbb{P}}^f_N$ detailed in Theorem~\ref{thm:main:result:inf:dim} below require that problem~\eqref{eq:def:Iprojection} admits a Slater point~$\mathbb{P}^\circ$ in the sense of the following assumption.

\begin{assumption}[Slater point] \label{ass:slater}
Problem~\eqref{eq:def:Iprojection} admits a Slater point $\mathbb{P}^\circ\in\Pi$ 
that satisfies
\begin{equation*}
 \textstyle{   \delta=\min_{y\not\in  E} \| \mathbb E_{\mathbb P^\circ}[ \psi(\xi)] - y\|_2 >0.}
\end{equation*}
\end{assumption}

Finding a Slater point~$\mathbb{P}^\circ$  may be difficult in general. However, $\mathbb{P}^\circ$ can be constructed systematically if $\psi$ is a polynomial  \cite[Remark~8]{ref:sutter-JMLR-19}, for example. Given~$\mathbb{P}^\circ$ and a tolerance~$\varepsilon>0$, we then define
\begin{align}
 &\textstyle{C=\D{\mathbb{P}^\circ}{\widehat{\mathbb{P}}_N}, \qquad D ={1 \over 2} \max_{y \in E} \|y\|_2, \qquad \eta_{1}=\frac{\varepsilon}{4D},  \qquad  \eta_{2}=\frac{\varepsilon \delta^2}{2C^2},} \nonumber \\
 &\textstyle{\alpha=\max_{\xi\in\Xi}\|\psi(\xi)\|_\infty, \quad L_\eta = 1/\eta_1 + \eta_2 + (\max_{\xi\in\Xi}\|\psi(\xi)\|_\infty)^2,} \nonumber \\
 &\textstyle{M_1(\varepsilon)=2 \left( \sqrt{\frac{8DC^2}{\varepsilon^2 \delta^2}+\frac{2\alpha^2 C^2}{\varepsilon \delta^2}+1}\right) \log\left(\frac{10(\varepsilon +2C)}{\varepsilon}\right), \label{eq:definitions:algo:cont}} \\
&\textstyle{M_2(\varepsilon)=2 \left(\! \sqrt{\frac{8DC^2}{\varepsilon^2 \delta^2}+\frac{2\alpha^2 C^2}{\varepsilon \delta^2}+1}\!\right)\! \log\!\left(\! \frac{C}{\varepsilon \delta(2-\sqrt{3})}\sqrt{4\left(\! \frac{4D}{\varepsilon}+\alpha^2 + \frac{\varepsilon \delta^2}{2C^2}\! \right)\!\left(\! C +\frac{\varepsilon}{2} \right)}\! \right).} \nonumber
\end{align}

\begin{theorem}[Almost linear convergence rate]\label{thm:main:result:inf:dim}
If Assumption~\ref{ass:slater} holds and~$\varepsilon>0$, then the candidate solution~\eqref{eq:estimates:primal:dual} obtained after~$k = \left \lceil \max\{ M_1(\varepsilon), M_2(\varepsilon)\} \right \rceil$ iterations of Algorithm~\hyperlink{algo:1}{1} satisfies
\begin{subequations} \label{eq:thm:error:bounds}
\begin{alignat}{3}
&\text{Optimality:}\hspace{15mm} &&|\D{\widehat{\mb Q}_{k}}{\widehat{\mathbb{P}}_N}- \D{\widehat{\mathbb{P}}^f_N}{\widehat{\mathbb{P}}_N}| \leq  2(1+2\sqrt{3})\varepsilon, \label{eq:thm:primal:optimality:cts} \\
&\text{Feasibility:}\hspace{10mm} &&\textstyle{\mathsf{d}\!\left(\mb{E}_{\widehat{\mb Q}_{k}}[\psi(\xi)]  ,E\right)\leq \frac{2\varepsilon\delta}{C},} \label{eq:thm:primal:feasibility:cts} 
\end{alignat}
\end{subequations}
where we use the definitions~\eqref{eq:definitions:algo:cont}, and the function~$\mathsf{d}(\cdot,E)$ denotes the Euclidean distance to the set~$E$ defined through $\mathsf{d}(x,E)=\min_{y\in E}\|x-y\|_2$. 
\end{theorem}
Theorem~\ref{thm:main:result:inf:dim} implies that Algorithm~\hyperlink{algo:1}{1} needs at most $O(\frac{1}{\varepsilon} \log \frac{1}{\varepsilon})$ iterations to find an $O(\varepsilon)$-suboptimal and $O(\varepsilon)$-feasible solution for the I-projection problem~\eqref{eq:def:Iprojection}. These results are derived via convex programming and duality by using the double smoothing techniques introduced in \cite{ref:devolder-12} and~\cite{ref:sutter-JMLR-19}.

\textbf{Computation of the DRO predictor.}
Equipped with Algorithm~\hyperlink{algo:1}{1} to efficiently approximate $\widehat{\mathbb{P}}_N^f$ via $\widehat{\mb Q}_{k}$, the DRO predictor $R^\star(\theta,\widehat{\mathbb{P}}^f_N)$ defined in \eqref{eq:def:DRO:general} can be approximated by $R^\star(\theta,\widehat{\mb Q}_{k})$ because the function $R^\star$ is continuous. We now show that the worst-case risk evaluation problem~\eqref{eq:def:DRO:predictor} admits a dual representation, which generalizes \cite[Proposition~5]{ref:vanParys:fromdata-17}.

\begin{proposition}[Dual representation of~$R^\star$]\label{prop:duality:DRO}
If~$r>0$, then the DRO predictor~$R^\star$ satisfies
\begin{equation} \label{eq:DRO:predictor:duality:formula}
R^\star(\theta,\mathbb{P}') = \left\{
\begin{array}{cl}
   \inf\limits_{\alpha\in\mathbb{R},z\in\mathbb{R}^d}  &  \alpha + \sigma_E(z) - e^{-r}\exp\left( \mathbb{E}_{\mathbb{P}'}[\log (\alpha - L(\theta,\xi)+z^\top \psi(\xi))]\right) \\
   \st & \alpha \geq \max_{\xi\in\Xi} L(\theta,\xi) - z^\top \psi(\xi) 
\end{array} \right.
\end{equation}
for ever~$\theta\in\Theta$ and $\mathbb{P}'\in\Pi$, where $\sigma_{E}(z)=\max_{x\in E} x^\top z$ denotes the support function of~$E$.
\end{proposition}
 Proposition~\ref{prop:duality:DRO} implies that if~$L(\theta,\xi)$ is convex in~$\theta$ for every~$\xi$, then the DRO predictor~\eqref{eq:def:DRO:predictor} coincides with the optimal value of a finite-dimensional convex program. Note that the objective function of~\eqref{eq:DRO:predictor:duality:formula} can be evaluated cheaply whenever the support function of~$E$ is easy to compute and~$\mathbb P'$ has finite support (e.g., if~$\mathbb P'$ is set to an output~$\widehat{\mb Q}_k$ of Algorithm~\hyperlink{algo:1}{1}). In addition, the robust constraint in~\eqref{eq:DRO:predictor:duality:formula} can be expressed in terms of explicit convex constraints if $L$, $\Xi$ and~$\psi$ satisfy certain regularity conditions. A trivial condition is that~$\Xi$ is finite. More general conditions are described in~\cite{ref:Hertog-15}.

\section{Experimental results} \label{sec:numerical:experiments}
We now assess the empirical performance of the MDI-DRO method in our two running examples.\footnote{All simulations were implemented in MATLAB and run on a 4GHz CPU with 16Gb RAM. The Matlab code for reproducing the plots is available from \url{https://github.com/tobsutter/PMDI_DRO}.}

\textbf{Synthetic dataset --- covariate shift adaptation.} The first two experiments revolve around the logistic regression problem with a distribution shift described in Example~\ref{ex:LR:part:1}. Specifically, we consider a synthetic dataset where the test data is affected by a covariate shift, which constitutes a special case of a distribution shift. Detailed information about the data generation process is provided in Appendix~\ref{app:numerics}. Our numerical experiments reveal that the proposed MDI-DRO method significantly outperforms the naive ERM method in the sense that its out-of-sample risk has both a lower mean as well as a lower variance; see Figures~\ref{fig:ERM:r:_2} and~\ref{fig:ERM:r:_4}. We also compare MDI-DRO against the IWERM method, which accounts for the distribution shift by assigning importance weights ${p^\star(\cdot)}/{p(\cdot)}$ to the training samples, where $p^\star(\cdot)$ and $p(\cdot)$ denote the densities of the test distribution $\mathbb P^\star$ and training distribution $\mathbb{P}$, respectively. These importance weights are assumed to be known in IWERM. In contrast, MDI-DRO does {\em not} require any knowledge of the test distribution other than its membership in~$\Pi$. Nevertheless, MDI-DRO displays a similar out-of-sample performance as IWERM even though it has less information about~$\mathbb P^\star$, and it achieves a lower variance than IWERM; see Figures~\ref{fig:IWERM:r:_2}-\ref{fig:IWERM:r:_4}. Figure~\ref{fig:cons:r_4} shows how the reliability of the upper confidence bound $J^\star_N$ and the out-of-sample risk $R(\theta^\star_N, \mathbb P^\star)$ change with the regularization parameter~$r$. Additional results are reported in Figure~\ref{fig:classification:appendix:synthetic} in the appendix. These results confirm that small regularization parameters~$r$ lead to small out-of-sample risk and that increasing~$r$ improves the reliability of the upper confidence bound~$J^\star_N$.
\begin{figure}[h!]
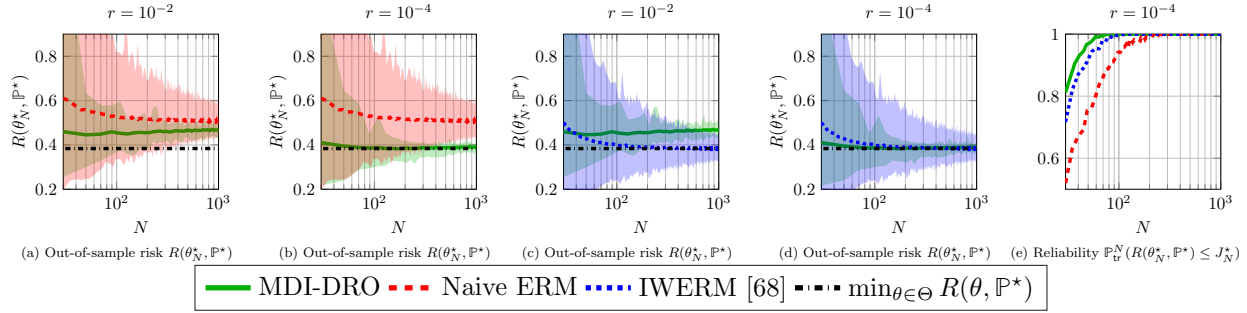
 
\centering
\scalebox{0.65}{
\subfloat[Out-of-sample risk $R(\theta^\star_N, \mathbb{P}^\star)$]{\input{oos_r_10_minus_2.tex} \label{fig:ERM:r:_2} }  
\hspace{2mm} 
\subfloat[Out-of-sample risk $R(\theta^\star_N, \mathbb{P}^\star)$]{\input{oos_r_10_minus_4.tex} \label{fig:ERM:r:_4} }
\newline
\subfloat[Out-of-sample risk $R(\theta^\star_N, \mathbb{P}^\star)$]{\input{oos_r_10_minus_2_part2.tex} \label{fig:IWERM:r:_2} }  
\hspace{2mm} 
\subfloat[Out-of-sample risk $R(\theta^\star_N, \mathbb{P}^\star)$]{\input{oos_r_10_minus_4_part2.tex} \label{fig:IWERM:r:_4} }
\subfloat[Reliability $\mathbb{P}^N_{\mathsf{tr}}( R(\theta^\star_N,\mathbb{P}^\star) \leq J^\star_N )$]{
%
%
\begin{tikzpicture}

\begin{axis}[%
width=1.25in,
height=1.25in,
at={(1.011111in,0.799444in)},
scale only axis,
xmode=log,
xmin=30,
xmax=1000,
xminorticks=true,
xlabel={$N$},
xmajorgrids,
xminorgrids,
ymin=0.5,
ymax=1,
ylabel={{\color{white}$\mathcal{R}(\widehat \beta_N, \mathbb{Q})$}},
y label style={yshift=-1.0em},
ymajorgrids,
title={$r=10^{-4}$},
legend to name=named,
legend style={legend cell align=left,align=left,draw=white!15!black,legend columns=4}
]

\addplot [color=darkgreen,solid,line width=2.0pt]
  table[row sep=crcr]{%
30	0.812\\
34	0.869\\
37	0.906\\
41	0.931\\
45	0.945\\
48	0.966\\
52	0.972\\
56	0.982\\
59	0.992\\
63	0.99\\
67	0.994\\
71	0.996\\
74	0.998\\
78	1\\
82	0.999\\
85	1\\
89	1\\
93	1\\
96	1\\
100	1\\
103	1\\
108	1\\
113	1\\
118	1\\
123	1\\
127	1\\
132	1\\
137	1\\
142	1\\
147	1\\
152	1\\
157	1\\
162	1\\
166	1\\
171	1\\
176	1\\
181	1\\
186	1\\
191	1\\
196	1\\
201	1\\
205	1\\
210	1\\
215	1\\
220	1\\
225	1\\
230	1\\
235	1\\
240	1\\
244	1\\
249	1\\
254	1\\
259	1\\
264	1\\
269	1\\
274	1\\
279	1\\
283	1\\
288	1\\
293	1\\
298	1\\
303	1\\
308	1\\
313	1\\
318	1\\
322	1\\
327	1\\
332	1\\
337	1\\
342	1\\
347	1\\
352	1\\
357	1\\
361	1\\
366	1\\
371	1\\
376	1\\
381	1\\
386	1\\
391	1\\
396	1\\
400	1\\
405	1\\
410	1\\
415	1\\
420	1\\
425	1\\
430	1\\
435	1\\
439	1\\
444	1\\
449	1\\
454	1\\
459	1\\
464	1\\
469	1\\
474	1\\
478	1\\
483	1\\
488	1\\
493	1\\
498	1\\
503	1\\
508	1\\
513	1\\
517	1\\
522	1\\
527	1\\
532	1\\
537	1\\
542	1\\
547	1\\
552	1\\
556	1\\
561	1\\
566	1\\
571	1\\
576	1\\
581	1\\
586	1\\
591	1\\
595	1\\
600	1\\
605	1\\
610	1\\
615	1\\
620	1\\
625	1\\
630	1\\
634	1\\
639	1\\
644	1\\
649	1\\
654	1\\
659	1\\
664	1\\
669	1\\
673	1\\
678	1\\
683	1\\
688	1\\
693	1\\
698	1\\
703	1\\
708	1\\
712	1\\
717	1\\
722	1\\
727	1\\
732	1\\
737	1\\
742	1\\
747	1\\
751	1\\
756	1\\
761	1\\
766	1\\
771	1\\
776	1\\
781	1\\
786	1\\
790	1\\
795	1\\
800	1\\
805	1\\
810	1\\
815	1\\
820	1\\
825	1\\
829	1\\
834	1\\
839	1\\
844	1\\
849	1\\
854	1\\
859	1\\
864	1\\
868	1\\
873	1\\
878	1\\
883	1\\
888	1\\
893	1\\
898	1\\
903	1\\
907	1\\
912	1\\
917	1\\
922	1\\
927	1\\
932	1\\
937	1\\
942	1\\
946	1\\
951	1\\
956	1\\
961	1\\
966	1\\
971	1\\
976	1\\
981	1\\
985	1\\
990	1\\
995	1\\
1000	1\\
};
\addlegendentry{MDI-DRO};


\addplot [color=red,dashed,line width=2.0pt]
  table[row sep=crcr]{%
30	0.52\\
35	0.632\\
40	0.671\\
45	0.692\\
49	0.754\\
54	0.762\\
59	0.799\\
64	0.837\\
69	0.855\\
74	0.885\\
79	0.889\\
84	0.91\\
88	0.911\\
93	0.925\\
98	0.941\\
103	0.941\\
108	0.958\\
113	0.958\\
118	0.967\\
123	0.955\\
127	0.979\\
132	0.986\\
137	0.974\\
142	0.982\\
147	0.98\\
152	0.984\\
157	0.983\\
162	0.979\\
166	0.989\\
171	0.987\\
176	0.994\\
181	0.984\\
186	0.991\\
191	0.99\\
196	0.995\\
201	0.996\\
205	0.998\\
210	0.997\\
215	0.995\\
220	1\\
225	0.993\\
230	0.997\\
235	0.998\\
240	0.998\\
244	0.996\\
249	0.999\\
254	1\\
259	0.999\\
264	0.998\\
269	1\\
274	1\\
279	0.999\\
283	0.999\\
288	1\\
293	1\\
298	1\\
303	0.999\\
308	1\\
313	1\\
318	1\\
322	0.999\\
327	1\\
332	1\\
337	1\\
342	1\\
347	1\\
352	0.999\\
357	1\\
361	1\\
366	1\\
371	1\\
376	1\\
381	1\\
386	1\\
391	1\\
396	1\\
400	1\\
405	1\\
410	1\\
415	1\\
420	1\\
425	1\\
430	1\\
435	1\\
439	1\\
444	1\\
449	1\\
454	1\\
459	1\\
464	1\\
469	1\\
474	1\\
478	1\\
483	1\\
488	1\\
493	1\\
498	1\\
503	1\\
508	1\\
513	1\\
517	1\\
522	1\\
527	1\\
532	1\\
537	1\\
542	1\\
547	1\\
552	1\\
556	1\\
561	1\\
566	1\\
571	1\\
576	1\\
581	1\\
586	1\\
591	1\\
595	1\\
600	1\\
605	1\\
610	1\\
615	1\\
620	1\\
625	1\\
630	1\\
634	1\\
639	1\\
644	1\\
649	1\\
654	1\\
659	1\\
664	1\\
669	1\\
673	1\\
678	1\\
683	1\\
688	1\\
693	1\\
698	1\\
703	1\\
708	1\\
712	1\\
717	1\\
722	1\\
727	1\\
732	1\\
737	1\\
742	1\\
747	1\\
751	1\\
756	1\\
761	1\\
766	1\\
771	1\\
776	1\\
781	1\\
786	1\\
790	1\\
795	1\\
800	1\\
805	1\\
810	1\\
815	1\\
820	1\\
825	1\\
829	1\\
834	1\\
839	1\\
844	1\\
849	1\\
854	1\\
859	1\\
864	1\\
868	1\\
873	1\\
878	1\\
883	1\\
888	1\\
893	1\\
898	1\\
903	1\\
907	1\\
912	1\\
917	1\\
922	1\\
927	1\\
932	1\\
937	1\\
942	1\\
946	1\\
951	1\\
956	1\\
961	1\\
966	1\\
971	1\\
976	1\\
981	1\\
985	1\\
990	1\\
995	1\\
1000	1\\
};
\addlegendentry{Naive ERM};


\addplot [color=blue,dotted,line width=2.0pt]
  table[row sep=crcr]{%
30	0.717\\
35	0.821\\
40	0.871\\
45	0.891\\
49	0.912\\
54	0.943\\
59	0.954\\
64	0.953\\
69	0.978\\
74	0.976\\
79	0.992\\
84	0.99\\
88	0.999\\
93	0.996\\
98	0.995\\
103	0.995\\
108	0.998\\
113	0.999\\
118	1\\
123	0.998\\
127	1\\
132	1\\
137	1\\
142	1\\
147	1\\
152	1\\
157	1\\
162	1\\
166	1\\
171	1\\
176	1\\
181	1\\
186	1\\
191	1\\
196	1\\
201	1\\
205	1\\
210	1\\
215	1\\
220	1\\
225	1\\
230	1\\
235	1\\
240	1\\
244	1\\
249	1\\
254	1\\
259	1\\
264	1\\
269	1\\
274	1\\
279	1\\
283	1\\
288	1\\
293	1\\
298	1\\
303	1\\
308	1\\
313	1\\
318	1\\
322	1\\
327	1\\
332	1\\
337	1\\
342	1\\
347	1\\
352	1\\
357	1\\
361	1\\
366	1\\
371	1\\
376	1\\
381	1\\
386	1\\
391	1\\
396	1\\
400	1\\
405	1\\
410	1\\
415	1\\
420	1\\
425	1\\
430	1\\
435	1\\
439	1\\
444	1\\
449	1\\
454	1\\
459	1\\
464	1\\
469	1\\
474	1\\
478	1\\
483	1\\
488	1\\
493	1\\
498	1\\
503	1\\
508	1\\
513	1\\
517	1\\
522	1\\
527	1\\
532	1\\
537	1\\
542	1\\
547	1\\
552	1\\
556	1\\
561	1\\
566	1\\
571	1\\
576	1\\
581	1\\
586	1\\
591	1\\
595	1\\
600	1\\
605	1\\
610	1\\
615	1\\
620	1\\
625	1\\
630	1\\
634	1\\
639	1\\
644	1\\
649	1\\
654	1\\
659	1\\
664	1\\
669	1\\
673	1\\
678	1\\
683	1\\
688	1\\
693	1\\
698	1\\
703	1\\
708	1\\
712	1\\
717	1\\
722	1\\
727	1\\
732	1\\
737	1\\
742	1\\
747	1\\
751	1\\
756	1\\
761	1\\
766	1\\
771	1\\
776	1\\
781	1\\
786	1\\
790	1\\
795	1\\
800	1\\
805	1\\
810	1\\
815	1\\
820	1\\
825	1\\
829	1\\
834	1\\
839	1\\
844	1\\
849	1\\
854	1\\
859	1\\
864	1\\
868	1\\
873	1\\
878	1\\
883	1\\
888	1\\
893	1\\
898	1\\
903	1\\
907	1\\
912	1\\
917	1\\
922	1\\
927	1\\
932	1\\
937	1\\
942	1\\
946	1\\
951	1\\
956	1\\
961	1\\
966	1\\
971	1\\
976	1\\
981	1\\
985	1\\
990	1\\
995	1\\
1000	1\\
};
\addlegendentry{IWERM \cite{JMLR:v8:sugiyama07a}};

\addplot [color=black, dashdotted,line width=2.0pt]
  table[row sep=crcr]{%
1	-1\\
1000	 -1\\
};
\addlegendentry{$\min_{\theta\in \Theta} R(\theta,\mathbb{P}^\star)$};

\end{axis}
\end{tikzpicture}%
\label{fig:cons:r_4}}}
\newline
\ref{named}
\caption[]{Results for a synthetic dataset with $m=6$. Shaded areas and lines represent ranges and mean values across 1000 independent experiments, respectively.}
\label{fig:toy:out-of-sample:high:D}
\end{figure}

\textbf{Real data --- classification under sample bias.} 
The second experiment addresses the heart disease classification task of Example~\ref{ex:LR:part:1} based on a real dataset\footnote{\url{https://www.kaggle.com/ronitf/heart-disease-uci}} consisting of~$N^\star$
i.i.d.\ samples from an unknown test distribution $\mathbb{P}^\star$.
To assess the effects of a distribution shift, 
we construct a biased training dataset $\{(\widehat x_1,\widehat y_1),\hdots,(\widehat x_N,\widehat y_N)\}$, $N< N^\star$, in which male patients older than 60 years are substantially over-represented. Specifically,  the~$N$ training samples are drawn randomly from the set of the 20\% oldest male patients. Thus, the training data follows a distribution $\mathbb{P}\neq \mb P^\star$. Even though the test distribution $\mathbb P^\star$ is unknown, we assume to know the empirical mean $m = \frac{1}{N^\star} \sum_{i=1}^{N^\star} (\widehat x_i,\widehat y_i)$ of the entire dataset to within an absolute error~$\Delta m>0$. The test distribution thus belongs to the set $\Pi$ defined in~\eqref{eq:Pi} with $E=[m - \Delta m \mathsf{1},m + \Delta m \mathsf{1}]$ and with~$\psi(x,y) = (x,y)$. 
We compare the proposed MDI-DRO method for classification against the naive ERM method that ignores the sample bias. In addition, we use a logistic regression model trained on the entire dataset as an (unachievable) ideal benchmark. Figure~2a shows the out-of-sample cost, Figure~2b the upper confidence bound~$J^\star_N$ and Figure~2c the misclassification rates of the different methods as the radius~$r$ of the ambiguity set is swept. Perhaps surprisingly, for some values of~$r$ the classification performance of MDI-DRO is comparable to that of the logistic regression method trained on the entire dataset.

\begin{figure}[h!] 
\begin{center}
   \includegraphics[width=0.85\textwidth]{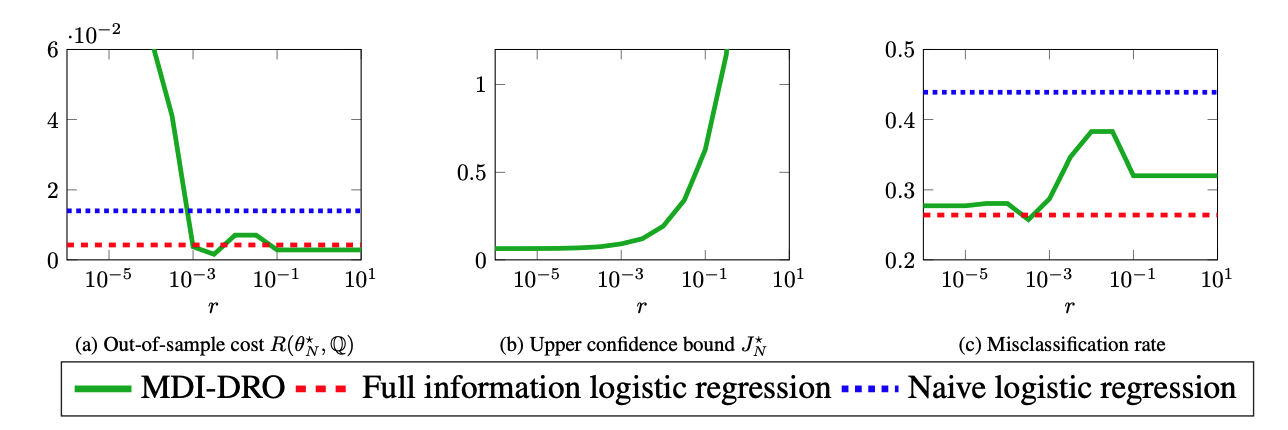}
   \end{center}
\caption[]{Heart disease classification example with $m=6$, $N=20$, $N^\star=303$ and $\Delta m=10^{-3}$.}
\label{fig:heart}
\end{figure}

\textbf{OPE for MDPs --- inventory control.}
We now consider the OPE problem of Examples~\ref{ex:OPE:part:1} and \ref{ex:OPE:part:2}.
A popular estimator for the cost $V_{\pi_{\mathsf{e}}}$ of the evaluation policy is the inverse propensity score (IPS)~\cite{ref:Rosenbaum-83} 
\begin{equation*} 
\textstyle{\widehat J_N^{\rm IPS} = \frac{1}{N}  \sum_{i=1}^N c(\widehat s_i,\widehat a_i)  \frac{\mu_\mathsf{e}(\widehat s_i,\widehat a_i)}{\mu_\mathsf{b}(\widehat s_i,\widehat a_i)}}.
\end{equation*}
Hoeffding's inequality then gives rise to the simple concentration bound
\begin{equation}\label{eq:hoeffding}
\textstyle{\mathbb{P}^N \left( V_{\pi_{\mathsf{e}}}\leq \widehat J_N^{\rm IPS} +\varepsilon \right) 
\geq 1 - e^{\frac{-2N\varepsilon^2}{b^2}}}\quad \forall \varepsilon>0,\;\forall N\in\mathbb{N},
\end{equation}
where $b=\max_{s\in\mathcal{S},a\in\mathcal{A}}c(s,a) \mu_\mathsf{e}(s,a)/\mu_\mathsf{b}(s,a)$. As~$b$ is typically a large constant, the finite sample bound~\eqref{eq:OPE:generalization} for~$J^\star_N$ is often more informative than \eqref{eq:hoeffding}. In addition, the variance of~$\widehat J_N^{\rm IPS}$ grows exponentially with the sample size~$N$~\cite{ref:Cortes-10,ref:Peters-13,ref:Strehl-10}. As a simple remedy, one can cap the importance weights beyond some threshold~$\beta>0$ and construct the modified IPS estimator as 
\begin{equation*}
\textstyle{\widehat J_N^{\mathop{\rm IPS}_\beta} = \frac{1}{N}  \sum_{i=1}^N  c(\widehat s_i, \widehat a_i) \min\left\{ \beta, \frac{\mu_\mathsf{e}(\widehat s_i,\widehat a_i)}{\mu_\mathsf{b}(\widehat s_i,\widehat a_i)} \right\}.}
\end{equation*}
Decreasing $\beta$ reduces the variance of $\widehat J_N^{\mathop{\rm IPS}_\beta}$ but increases its bias. An alternative estimator for~$V_{\pi_\mathsf{e}}$ is the doubly robust (DR) estimator $\widehat J_N^{\rm DR}$, which uses a control variate to reduce the variance of the IPS estimator. The DR estimator was first developed for contextual bandits \cite{ref:Dudik-14} and then generalized to MDPs~\cite{ref:Jiang-16, ref:Tang-20}. We evaluate the performance of the proposed MDI-DRO estimator on a classical inventory control problem. A detailed problem description is relegated to Appendix~\ref{app:numerics}. We sample both the evaluation policy $\pi_\mathsf{e}$ and the behavioral policy $\pi_\mathsf{b}$ from the uniform distribution on the space of stationary policies. The decision maker then has access to the evaluation policy $\pi_\mathsf{e}$ and to a sequence of i.i.d.~state action pairs $\{\widehat s_i,\widehat a_i\}_{i=1}^N$ sampled from~$\mu_\mathsf{b}$ as well as the observed empirical costs~$\{\widehat c_i\}_{i=1}^N$, where $\widehat c_i = c(\widehat s_i, \widehat a_i)$. 
Figure~\ref{fig:OPE:inventory} compares the proposed MDI-DRO estimator against the original and modified IPS estimators, the DR estimator and the ground truth expected cost of the evaluation policy. Figures~3a and~3b show that for small radii~$r$, the MDI-DRO estimator outperforms the IPS estimators both in terms of accuracy and precision. 
Figure~3c displays the disappointment probabilities $\mathbb{P}^N(V_{\pi_\mathsf{e}}>\widehat J_N)$ analyzed in Theorem~\ref{thm:admissibility}, where~$\widehat J_N$ denotes any of the tested estimators.

\begin{figure}[h!] 
\centering
\includegraphics[width=0.75\textwidth]{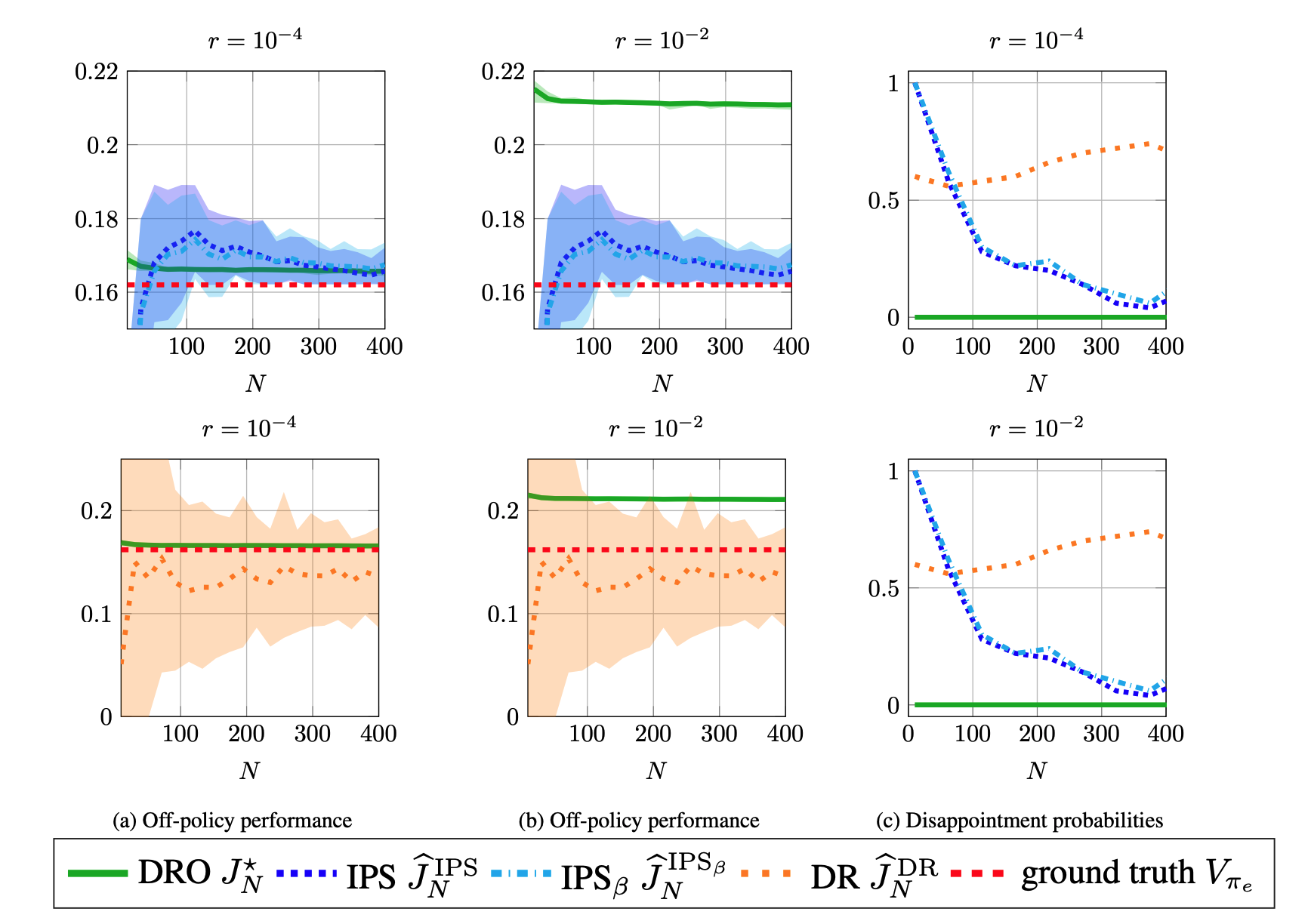}
\caption[]{Shaded areas and lines represent 90\% confidence intervals and mean values across 1000 independent experiments, respectively.}
\label{fig:OPE:inventory}
\end{figure}

\textbf{Acknowledgments.} This research was supported by the Swiss National Science Foundation under the NCCR Automation, grant agreement 51NF40\_180545.

\section{Appendix} \label{sec:appendix}
The appendix details all proofs and provides some auxiliary results grouped by section.
\subsection{Proofs of Section~\ref{sec:problem:statement}} \label{app:secc:proofs:problem:statement}

\begin{proof}[Proof of Proposition~\ref{prop:conditional:limit:thm}]
Denote by $\mathbb{P}^N_{\xi_1|\Pi}$ the probability distribution of $\xi_1$ with respect to~$\mathbb P^N$ conditional on the event $\widehat{\mathbb{P}}_N\in\Pi$. By \cite[Theorem~4]{ref:Csiszar-84}, we then have
 \begin{equation*}
\lim_{N\to\infty} \D{\mathbb{P}^N_{\xi_1|\Pi}}{\mathbb{P}^{f}}=0,
\end{equation*}
i.e., the conditional distribution $\mathbb{P}^N_{\xi_1|\Pi}$ converges in information to~$\mathbb{P}^{f}$. As the moment-generating function $\mathbb E_{\mathbb P^f}[ e^{t L(\theta,\xi)}]$ is finite for all $t$ in a neighborhood of~$0$,
\cite[Lemma~3.1]{csiszar1975} ensures that 
 \begin{equation*}
\lim_{N\to\infty} \mathbb{E}_{\mathbb{P}^N}[L(\theta,\xi_1)|\widehat{\mathbb{P}}_N\in \Pi] 
= \lim_{N\to\infty} \mathbb{E}_{\mathbb{P}^N_{\xi_1|\Pi}}[L(\theta,\xi_1)]
= \mathbb{E}_{\mathbb{P}^\star}[L(\theta,\xi_1)].
\end{equation*}
Thus, the claim follows.
\end{proof}

\begin{proof}[Proof of Proposition~\ref{prop:every:measure:I:projection}] Proposition~\ref{prop:every:measure:I:projection} can be seen as a generalization of \cite[Exercise~12.6]{cover2006elements}. To simplify notation, we define $\alpha = \D{\mathbb{Q}}{\mathbb{P}}$. Then, we have
\begin{subequations}
\begin{align}
 \min\limits_{\bar{\mathbb{Q}}\in\Pi}  \D{\bar{\mathbb{Q}}}{\mathbb{P}}
 &\label{eq:pf:max:ent:step1}=\min\limits_{\bar{\mathbb{Q}}\in\mathcal{P}(\Xi)} \sup_{\lambda\in\Re}  \D{\bar{\mathbb{Q}}}{\mathbb{P}} - \lambda \left( \int_\Xi \log\left(\frac{\d \mathbb{Q}}{\d \mathbb{P}}\right) \d\bar{\mathbb{Q}} - \alpha \right) \\
 &\label{eq:pf:max:ent:step2}=\max_{\lambda\in\Re} \min\limits_{\bar{\mathbb{Q}}\in\mathcal{P}(\Xi)}  \D{\bar{\mathbb{Q}}}{\mathbb{P}} - \lambda  \int_\Xi \log\left(\frac{\d\mathbb{Q}}{\d \mathbb{P}}\right) \d\bar{\mathbb{Q}} + \lambda \alpha \\
 &\label{eq:pf:max:ent:step3}=\max_{\lambda\in\Re} - \log \int_\Xi \left( \frac{\d\mathbb{Q}}{\d\mathbb{P}} \right)^\lambda \d\mathbb{P}+ \lambda \alpha =\alpha, 
\end{align}
\end{subequations}
where \eqref{eq:pf:max:ent:step1} holds by the definition of the set~$\Pi$, and~\eqref{eq:pf:max:ent:step2} follows from Sion's minimax theorem. The latter applies because the relative entropy $\D{\bar{\mathbb{Q}}}{\mathbb{P}}$ is convex in~$\bar{\mathbb{Q}}$ and the distribution family~$\mathcal{P}(\Xi)$ is convex and weakly compact thanks to the compactness of~$\Xi$. Finally, \eqref{eq:pf:max:ent:step3} holds because of \cite[Lemma~2]{ref:sutter-JMLR-19}, which implies that the inner minimization problem in~\eqref{eq:pf:max:ent:step2} is uniquely solved by the probability distribution~$\bar{\mathbb{Q}}^\star_\lambda\in \mathcal{P}(\Xi)$ defined through
\begin{equation*}
\bar{\mathbb{Q}}^\star_\lambda(B) = \frac{\int_B e^{\lambda \log \left( \frac{\d\mathbb{Q}}{\d \mathbb{P}} \right)}\d \mathbb{P}}{\int_\Xi e^{\lambda \log \left( \frac{\d\mathbb{Q}}{\d \mathbb{P}} \right)}\d \mathbb{P}}
=\frac{\int_B \left( \frac{\d\mathbb{Q}}{\d \mathbb{P}} \right)^\lambda \d \mathbb{P}}{\int_\Xi \left( \frac{\d\mathbb{Q}}{\d \mathbb{P}} \right)^\lambda \d \mathbb{P}} \quad \forall B\in\mathcal{B}(\Xi).
\end{equation*}
By inspecting the first-order optimality condition of the convex maximization problem in \eqref{eq:pf:max:ent:step3} and remembering that $\alpha = \D{\mathbb{Q}}{\mathbb{P}}$, one can then show that \eqref{eq:pf:max:ent:step3} is solved by $\lambda^\star=1$. The Nash equilibrium of the zero-sum game in~\eqref{eq:pf:max:ent:step2} is therefore given by~$\lambda^\star$ and its unique best response~$\bar{\mathbb{Q}}^\star_{\lambda^\star}=\mathbb Q$, and the solution~$f(\mathbb P)$ of the I-projection problem in~\eqref{eq:pf:max:ent:step1} coincides with~$\mathbb{Q}$.
\end{proof}

\subsection{Proofs of Section~\ref{sec:stat:guarantees}}

\begin{proof}[Proof of Theorem~\ref{thm:admissibility}]
The continuity of~$R^\star$ on~$\Theta\times\Pi$ is established in Corollary~\ref{corollary:dual:multiplier:bound} below.

In order to prove that the DRO predictor~$R^\star$ is also admissible, we first prove that the following inequality holds for any fixed $\theta\in\Theta$ and $\mathbb{P}\in\mathcal{P}(\Xi)$.
\begin{align} \label{eq:pf:admiss:claim1}
\limsup_{N\to\infty}\frac{1}{N}\log \mathbb{P}^N & \left( R(\theta,\mathbb{P}^f) > R^\star(\theta,\widehat{\mathbb{P}}^f_N) \right) \leq -r 
\end{align}
For the sake of concise notation, we then define the disappointment set 
\[
    \mc D(\theta,\mathbb{P})=\left\{\mathbb{P}'\in\mathcal{P}(\Xi) : R(\theta,f(\mathbb{P})) > R^\star(\theta,f(\mathbb{P}'))\right\}
\]
containing all realizations~$\mathbb P'$ of the empirical distribution~$\widehat {\mathbb P}_N$, for which the true risk~$R(\theta,f(\mathbb{P}))$ under the I-projection of the unknown true distribution exceeds the risk~$R^\star(\theta,f(\mathbb{P}'))$ predicted by the distributionally robust predictor under the I-projection of the empirical distribution. Hence, $\bar{\mc D}(\theta,\mathbb{P})$ contains all realizations of~$\widehat {\mathbb P}_N$ under which the distributionally robust predictor is too optimistic and thus leads to disappointment. Similarly, we define the weak disappointment set
\[
    \bar{\mc D} (\theta,\mathbb{P})=\left\{\mathbb{P}'\in\mathcal{P}(\Xi) : R(\theta,f(\mathbb{P})) \geq R^\star(\theta,f(\mathbb{P}'))\right\},
\]
which simply replaces the strict inequality in the definiton of~$\bar{\mc D} (\theta,\mathbb{P})$ with a weak inequality. 
Recall now that $R^\star$ is continuous. In addition, note that $f$ is continuous thanks to \cite[Theorem~9.17]{sundaram_1996}, which follows from the strict convexity of the relative entropy in its first argument \cite[Lemma~6.2.12]{dembo2009large}. Therefore the set $\bar{\mc D}(\theta,\mathbb{P})$ is closed, and $\cl  \mc D(\theta,\mathbb{P}) \subset \bar{\mc D}(\theta,\mathbb{P})$. The left hand side of~\eqref{eq:pf:admiss:claim1} thus satisfies
\begin{align*}
\limsup_{N\to\infty}\frac{1}{N}\log \mathbb{P}^N \left( R(\theta,f(\mathbb{P})) > R^\star(\theta,f(\widehat {\mathbb{P}}_N)) \right)
&=\limsup_{N\to\infty}\frac{1}{N}\log \mathbb{P}^N\left( \widehat{\mathbb{P}}_N \in \mc D(\theta,\mathbb{P}) \right)\\
&\leq - \inf_{\mathbb{P}'\in\cl \mc D(\theta,\mathbb{P})} \D{\mathbb{P}'}{\mathbb{P}} \\
&\leq - \inf_{\mathbb{P}'\in \bar{\mc D} (\theta,\mathbb{P})} \D{\mathbb{P}'}{\mathbb{P}} \\
&\leq -r,
\end{align*}
where the first inequality follows from Sanov's Theorem, which asserts that $\widehat{\mathbb{P}}_N$ satisfies a large deviation principle with the relative entropy as the rate function \cite[Theorem~6.2.10]{dembo2009large}.
The second inquality exploits the inclusion $\cl  \mc D(\theta,\mathbb{P}) \subset \bar{\mc D}(\theta,\mathbb{P})$, and the last inequality holds because
\begin{equation*}
\mathbb{P}'\in \bar{\mc D}(\theta,\mathbb{P})\quad\implies\quad  \D{f(\mathbb{P}')}{f(\mathbb{P})} \geq r
\quad \implies \quad \D{\mathbb{P}'}{\mathbb{P}} \geq r,
\end{equation*}
where the first implication has been established in te proof of \cite[Theorem~10]{ref:vanParys:fromdata-17}, and the second implication follows from the data-processing inequality \cite[Lemma~3.11]{ref:Csiszar_Koerner-82}. This proves \eqref{eq:pf:admiss:claim1}.

In the last step of the proof, we fix an arbitrary~$\varepsilon>0$ and show that 
\begin{align*}
\limsup_{N\to\infty}\frac{1}{N}\log \mathbb{P} & \left( R(\theta^\star_N,\mathbb{P}^f) > R^\star(\theta^\star_N,\widehat{\mathbb{P}}^f_N)+\varepsilon \right) \leq -r
\end{align*}
for any~$\mathbb P\in\mathcal P(\Xi)$, where $\theta^\star_N$ is defined as usual as a minimizer of~\eqref{eq:def:DRO:general}. The proof of this generalized statement widely parallels that of \cite[Theorem~11]{ref:vanParys:fromdata-17} and exploits the data processing inequality in a similar manner as in the proof of \eqref{eq:pf:admiss:claim1}. Details are omitted for brevity.
\end{proof}

\begin{proof}[Proof of Theorem~\ref{thm:statistical:efficiency}]
The proof is inspired by \cite[Theorems~7 \&~11]{ref:vanParys:fromdata-17}. We first show that any continuous admissible data-driven predictor $\widehat R$ satisfies the inequality
\begin{equation}
    \label{eq:optimality-predictor}
    \lim_{N\to\infty} R^\star(\theta,f(\widehat{\mb P}_N)) \leq \lim_{N\to\infty} \widehat R(\theta,f(\widehat{\mb P}_N)) \quad \mb P^\infty\text{-a.s.}
\end{equation}
for all $\theta\in\Theta$ and $\mathbb{P}\in\mathcal{P}(\Xi)$. As the empirical distribution~$\widehat{\mathbb P}_N$ converges weakly to~$\mathbb P$ and as $R^\star$, $\widehat R$ and~$f$ represent continuous mappings, the inequality \eqref{eq:optimality-predictor} is equivalent to
\begin{equation*}
      R^\star(\theta,f(\mb P)) \leq \widehat R(\theta,f(\mb P))
\end{equation*}
for all $\theta\in\Theta$ and $\mathbb{P}\in\mathcal{P}(\Xi)$. Suppose now for the sake of contradiction there exists a continuous admissible predictor $\widehat R$,
a parameter $\theta_0\in \Theta$ and an asymptotic estimator realization $\mb P'_0\in\mc P(\Xi)$ with
\begin{equation*}\label{eq:proof:thm:contrad:optimality}
\widehat R(\theta_0,f(\mb P'_0)) < R^\star(\theta_0,f(\mb P'_0)) .
\end{equation*}
In fact, as~$\widehat R$, $R^\star$ and~$f$ are continuous functions, this strict inequality holds on a neighborhood of~$\mathbb P_0'$. 
Next, define $\varepsilon = R^\star(\theta_0,f(\mb P'_0)) - \widehat R(\theta_0,f(\mb P'_0))>0$ and denote by $\bar{\mb P}\in\Pi$ an optimizer of the worst-case risk evaluation problem~\eqref{eq:def:DRO:predictor} for~$\mathbb P'=f(\mathbb P'_0)$, which satisfies~$R^\star(\theta_0,f(\mb P'_0)) = R(\theta_0,\bar{\mb P})$ and~$\D{f(\mb P'_0)}{\bar{\mb P}}\leq r$. By using a continuity argument as in the proof of~\cite[Theorem~10]{ref:vanParys:fromdata-17} and by exploiting the convexity of~$\Pi$, one can then show that there exists a model~$\mb P_0\in\Pi$ with
\begin{equation} \label{eq:proof:prop:rm:point}
R(\theta_0,\bar{\mb P})<R(\theta,\mb P_0)+\varepsilon \quad \text{and}\quad \D{f(\mb P'_0)}{\mb P_0}=r_0<r.
\end{equation}
All of this implies that
\begin{equation}\label{eq:proof:prop:rm:point:2}
\widehat R(\theta_0,f(\mb P'_0))  = R^\star(\theta_0,f(\mb P'_0)) - \varepsilon = R(\theta_0,\bar{\mb P}) - \varepsilon <R(\theta_0,\mb P_0)=R(\theta_0,f(\mb P_0)),
\end{equation}
where the three equalities follow from the definition of~$\varepsilon$, the construction of~$\bar{\mathbb P}$ and the observation that~$f$ reduces to the identity mapping when restricted to~$\Pi$. The inequality holds due to the first relation in~\eqref{eq:proof:prop:rm:point}. In analogy to the proof of Theorem~\ref{thm:admissibility}, we now introduce the  disappointment set for the data-driven predictor~$\widehat R$ under the data-generating distribution~$\mb P_0$, that is,
\begin{equation*}
\mathcal{D}(\theta_0,\mb P_0)=\left\{\mathbb{P}'\in\mathcal{P}(\Xi) : R(\theta_0,f(\mb P_0))>\widehat R(\theta_0,f(\mathbb{P}')) \right\}.
\end{equation*}
The relation \eqref{eq:proof:prop:rm:point:2} readily implies that $\mathbb{P}'_0\in\mathcal{D}(\theta_0,\mb P_0)$. As the I-projection is idempotent (that is, $f\circ f=f$), one can further verify that~$f(\mb P'_0)\in\mathcal{D}(\theta_0,\mb P_0)$. 
Denoting the empirical distribution of~$N$ training samples drawn independently from~$\mb P_0$ by $\widehat{\mb P}_{0,N}$, we thus find
\begin{equation*}\label{eq:optimality:contradiction}
\begin{aligned}
\liminf_{N\to\infty}\frac{1}{N} \log \mb P_0^N\left(R(\theta_0,f(\mb P_0)) >\widehat R (\theta_0,f(\widehat{\mb P}_{0,N}))\right) 
& = \liminf_{N\to\infty}\frac{1}{N} \log \mb P_0^N\left(\widehat{\mb P}_{0,N}\in \mathcal{D}(\theta_0,\mb P_0)\right) \\
& \geq -\inf_{\mathbb{P}'\in \interior \mathcal{D}(\theta_0,\mb P_0)} \D{\mathbb{P}'}{\mb P_0}\\
& = -\inf_{\mathbb{P}'\in \mathcal{D}(\theta_0,\mb P_0)} \D{\mathbb{P}'}{\mb P_0}\\
& \geq - \D{f(\mb P'_0)}{\mb P_0} \\
& = -r_0 > -r,
\end{aligned}
\end{equation*}
where the first inequality follows from Sanov's Theorem, which ensures that $\widehat{\mathbb{P}}_N$ satisfies a large deviation principle with the relative entropy as the rate function. The second equality holds because~$\mathcal{D}(\theta_0,\mb P_0')$ is open thanks to the continuity of~$\widehat R$ and~$f$, and the second inequality exploits our earlier insight that $f(\mathbb{P}'_0)\in\mathcal{D}(\theta_0,\mb P_0')$. The last inequality, finally, follows from the second relation in~\eqref{eq:proof:prop:rm:point}. The above reasoning shows that~$\widehat R$ fails to be admissible, and hence a data-driven predictor~$\widehat R$ with the advertised properties cannot exist. Thus, $R^\star$ indeed satisfies the efficiency property~\eqref{eq:optimality-predictor}.

To show that $\lim_{N\rightarrow\infty} J^\star_N\leq \lim_{N\rightarrow\infty} \widehat J_N$ $\mathbb P^\infty$-almost surely for all~$\mathbb P\in\mathcal P(\Xi)$, we use~\eqref{eq:optimality-predictor} and adapt the proof of \cite[Theorem~11]{ref:vanParys:fromdata-17} with obvious modifications. Details are omitted for brevity.
\end{proof}
\begin{proof}[Proof of Corollary~\ref{cor:finite-sample-guarantee}]
Recalling that Sanov's Theorem for finite state spaces offers finite sample bounds \cite[Theorem~11.4.1]{cover2006elements}, the claim can be established by repeating the proof of Theorem~\ref{thm:admissibility}.
\end{proof}

\begin{proof}[Proof of Theorem~\ref{thm:asymptotic:consistency}]

It suffices to prove~\eqref{eq:thm:consistency-assertion2} because~\eqref{eq:thm:consistency-assertion1} can be seen as a special case of~\eqref{eq:thm:consistency-assertion2} when~$\Theta=\{\theta\}$. In the remainder we denote by $\mathsf{d}_{\mathsf{TV}}(\mb P,\mb Q)$ the total variation distance and by $\mathsf{d}_{\mathsf{W_p}}(\mb P,\mb Q)$ the $p$-th Wasserstein distance ($p\in\mathbb{N}$) between two probability distributions $\mb P,\mb Q\in\mathcal P(\Xi)$. 
To make its dependence on the radius~$r$ explicit, throughout this proof we temporarily use~$R^\star_r$ to denote the DRO predictor~\eqref{eq:def:DRO:predictor}. As usual, we use~$\theta^\star_N\in\Theta$ to denote a minimizer of the DRO problem~\eqref{eq:def:DRO:general} with~$\mathbb P'=\widehat{\mathbb P}_N^f$. In addition, we use~$\widehat{\mathbb{Q}}^\star_{N,\theta}\in\Pi$ to denote a maximizer of the worst-case risk evaluation problem~\eqref{eq:def:DRO:predictor} with~$\mathbb P'=\widehat{\mathbb P}_N^f$. By definition, this maximizer must satisfy the relations
\begin{equation*}
    R^\star_{r_N} (\theta,\widehat{\mathbb{P}}_N^f)=R(\theta,\widehat{\mathbb{Q}}^\star_{N,\theta})\quad \text{and} \quad \D{\widehat{\mathbb{P}}^f_{N}}{\widehat{\mathbb{Q}}^\star_{N,\theta}}\leqslant r_N
\end{equation*}
for all~$\theta\in\Theta$ and~$N\in\mathbb N$. Pinsker's inequality then implies that
\begin{equation}\label{eq:Pinsker:ineq}
  \sup_{\theta\in \Theta} \mathsf{d}_{\mathsf{TV}}\left(\widehat{\mathbb{P}}^f_{N}, \widehat{\mathbb{Q}}^\star_{N,\theta} \right)
  \leq \sup_{\theta\in \Theta} \sqrt{\frac{1}{2}\D{\widehat{\mathbb{P}}^f_{N}}{\widehat{\mathbb{Q}}^\star_{N,\theta}}}
  \leq \sqrt{\frac{r_N}{2}} \quad \forall N\in\mb N.
\end{equation}
Thus, we find
\begin{align*}
    &\sup_{\theta\in\Theta} \left\{ \left| R^\star_{r_N} (\theta,\widehat{\mathbb{P}}_N^f) - R(\theta,\mathbb{P}^f)\right| \right\} \\
    &\qquad = \sup_{\theta\in\Theta}\left\{\left| \mathbb{E}_{\widehat{\mathbb{Q}}^\star_{N,\theta}}[L(\theta,\xi)] - \mathbb{E}_{\mathbb{P}^f}[L(\theta,\xi)]  \right|\right\} \\
    &\qquad \leq \sup_{\theta\in\Theta} \left\{ \left| \mathbb{E}_{\widehat{\mathbb{Q}}^\star_{N,\theta}}[L(\theta,\xi)] - \mathbb{E}_{\widehat{\mathbb{P}}_N^f}[L(\theta,\xi)]  \right| 
    +\left| \mathbb{E}_{\widehat{\mathbb{P}}_N^f}[L(\theta,\xi)] - \mathbb{E}_{\mathbb{P}^f}[L(\theta,\xi)]  \right|\right\}\\
    &\qquad \leq \sup_{\theta\in\Theta} \left\{ \left| \mathbb{E}_{\widehat{\mathbb{Q}}^\star_{N,\theta}}[L(\theta,\xi)] - \mathbb{E}_{\widehat{\mathbb{P}}_N^f}[L(\theta,\xi)]  \right| \right\}
    +\sup_{\theta\in\Theta} \left\{\left| \mathbb{E}_{\widehat{\mathbb{P}}_N^f}[L(\theta,\xi)] - \mathbb{E}_{\mathbb{P}^f}[L(\theta,\xi)]  \right|\right\}\\
    &\qquad \leq \Lambda \sup_{\theta\in\Theta}  \mathsf{d}_{\mathsf{W}_1}\left( \widehat{\mathbb{Q}}^\star_{N,\theta},\widehat{\mathbb{P}}_N^f\right)
    + \Lambda   \mathsf{d}_{\mathsf{W}_1}\left(\widehat{\mathbb{P}}_N^f,\mathbb{P}^f\right)\\
    &\qquad \leq \Lambda C \sup_{\theta\in\Theta}  \mathsf{d}_{\mathsf{TV}}\left( \widehat{\mathbb{Q}}^\star_{N,\theta},\widehat{\mathbb{P}}_N^f\right)
    + \Lambda  \mathsf{d}_{\mathsf{W}_2}\left(\widehat{\mathbb{P}}_N^f,\mathbb{P}^f\right),
\end{align*}
where the first three inequalities follow from the triangle inequality, the subadditivity of the supremum operator and the Kantorovich-Rubinstein theorem \cite[Theorem~5.10]{ref:Villani-08}, respectively. The last inequality holds because~$\Xi$ is compact, which implies that the first Wasserstein distance can be bounded above by the total variation distance scaled with a positive constant~$C$ \cite[Theorem~6.15]{ref:Villani-08} and because $\mathsf{d}_{\mathsf{W}_1}(\cdot,\cdot) \leq \mathsf{d}_{\mathsf{W}_2}(\cdot,\cdot)$ thanks to Jensen's inequality. By~\eqref{eq:Pinsker:ineq}, the first term in the above expression decays deterministically to zero as~$N$ grows. 
The second term converges $\mathbb{P}^\infty$-almost surely to zero as~$N$ increases because the empirical distribution converges $\mathbb{P}^\infty$-almost surely to the data-generating distribution in the second Wasserstein distance \cite{HOROWITZ1994261}. In summary, we thus have \begin{equation}\label{eq:uniform:conv:Q}
    \lim_{N\to\infty}\sup_{\theta\in\Theta} \left| R^\star_{r_N} (\theta,\widehat{\mathbb{P}}_N^f) - R(\theta,\mathbb{P}^f)\right| = 0 \quad \mathbb{P}^\infty\text{-a.s.}
\end{equation}
Put differently, for $\mathbb P^\infty$-almost every trajectory of training samples, the functions~$R^\star_{r_N} (\cdot,\widehat{\mathbb{P}}_N^f)$ converge uniformly to~$R(\cdot,\mathbb{P}^f)$. The claim then follows from~\cite[Proposition~7.15 and Theorem~7.31]{rockafellar1998variational}.

\end{proof}

\subsection{Proofs and auxiliary results for Section~\ref{sec:computation}} 

\begin{proof}[Proof of Theorem~\ref{thm:main:result:inf:dim}]
The key enabling mechanism to prove~\eqref{eq:thm:primal:optimality:cts} and \eqref{eq:thm:primal:feasibility:cts} is the so-called double smoothing method for linearly constrained convex programs \cite{ref:devolder-12}. Our proof parallels that of \cite[ Theorem~5]{ref:sutter-JMLR-19} and is provided here to keep the paper self contained. 
Throughout the proof, we denote by $\mc M(\Xi)$ the vector space of all finite signed Borel measures on~$\Xi$, and we equip~$\mc M(\Xi)$ with the total variation norm~$\|\cdot\|_{\mathsf{TV}}$. Choosing the total variation norm has the benefit that the  function~$g:\mc P(\Xi)\rightarrow\Re_+$ defined through $g(\mb Q) = \D{\mb Q}{\widehat{\mb P}_N}$ is strongly convex with convexity parameter~$1$. Indeed, Pinsker's inequality implies that~$d(\mb Q) \geq \frac{1}{2}\|\mb Q - \widehat{\mb P}_N\|_{\mathsf{TV}}^2$ for all $\mb Q\in\mc P(\Xi)$. To prove~\eqref{eq:thm:primal:optimality:cts} and~\eqref{eq:thm:primal:feasibility:cts}, we consider the primal and dual optimization problems
\begin{subequations}\label{eq:primal:dual:pairs}
\begin{align}
J_{\mathsf{P}}^{\star} &= \min\limits_{\mb Q\in\mathcal{P}(\Xi)} \Big \{  \D{\mb Q}{\widehat{\mathbb{P}}_N} + \sup_{z\in\Re^{d}}\left\{ \mb E_{\mb Q}[\psi(\xi)]^\top z- \sigma_{E}(z)\right\} \Big \} \label{eq:primal:problem}\\
J_{\mathsf{D}}^{\star} &= \sup_{z\in\Re^{d}} \Big \{ - \sigma_{E}(z)  +   \min\limits_{\mb Q\in\mathcal{P}(\Xi)} \left\{ \D{\mb Q}{\widehat{\mathbb{P}}_N}  + \mb E_{\mb Q}[\psi(\xi)]^\top z\right\} \Big \} \, ,  \label{eq:dual:problem}
\end{align}
\end{subequations}
where $\sigma_E:\Re^{d}\to \mathbb R$ defined through $\sigma_{E}(z)=\max_{x\in E}z^\top x$ denotes the support function of~$E$. As the convex conjugate of the support function~$\sigma_E$ is the indicator function~$\delta_E:\Re^{d}\to [0,\infty]$ defined through~$\delta_E(x)=0$ if~$x\in E$ and $\delta_E(x)=\infty$ if~$x\notin E$, the optimal value of the maximization problem over~$z$ in~\eqref{eq:primal:problem} equals~$\delta_E(\mb E_{\mb Q}[\psi(\xi)])$. Hence, the unique minimizer of~\eqref{eq:primal:problem} coincides with the I-projection of the empirical distribution onto the set~$\Pi$. We also remark that~$\sigma_E$ is continuous because~$E$ is non-empty and compact \cite[Corollary~13.2.2]{ref:Rockafellar-97}. Assumption~\ref{ass:slater} then ensures via~\cite[Lemma~3]{ref:sutter-JMLR-19} that there is no duality gap, i.e, $J_{\mathsf{P}}^{\star}=J_{\mathsf{D}}^{\star}$. Next, we introduce the shorthand
\begin{equation*} 
F(z)= - \sigma_{E}(z)  +   \min\limits_{\mb Q\in\mathcal{P}(\Xi)} \left\{ \D{\mb Q}{\widehat{\mathbb{P}}_N}  + \mb E_{\mb Q}[\psi(\xi)]^\top z\right\}
\end{equation*}
for the dual objective function. While the primal problem \eqref{eq:primal:problem} is an infinite-dimensional optimization problem, the dual problem \eqref{eq:dual:problem} can be solved via first-order methods provided that the gradient of the dual objective function~$F$ can be evaluated at low cost. Unfortunately, this function fails to be smooth. Consequently, an optimal first-order method would require $O(1/\varepsilon^2)$ iterations, where~$\varepsilon$ denotes the desired additive accuracy \cite[Section~3.2]{ref:Nesterov:book:14}. However, the computation can be accelerated by smoothing the dual objective function as in~\cite{ref:devolder-12, nesterov05} and by exploiting structural properties. To this end, we introduce a smoothed version~$F_\eta$ of the dual objective function defined through 
\begin{equation*} 
F_{\eta}(z)=  -\max_{x\in E}\left\{ x^\top z - \frac{\eta_{1}}{2}\norm{x}_{2}^{2} \right\} +  \min\limits_{\mb Q\in\mathcal{P}(\Xi)} \left\{ \D{\mb Q}{\widehat{\mathbb{P}}_N} + \mb E_{\mb Q}[\psi(\xi)]^\top z \right\}-\frac{\eta_{2}}{2}\norm{z}_{2}^{2} \, ,
\end{equation*}
where $\eta=(\eta_{1},\eta_{2})\in\Re_{++}^{2}$ is a smoothing parameter. One readily verifies that~$x^\star_z=\pi_E(\eta_1^{-1}z)$ solves the optimization problem in the first term. The optimization problem in the second term minimizes the sum of a relative entropy function and a linear function. Therefore, it is reminiscent of an entropy maximization problem, and one can show that it is solved by the Gibbs distribution
\begin{equation*}
    \mb Q_z^\star =\frac{\sum_{j=1}^N\exp\left(-z^\top \psi\,(\widehat\xi_j)\right) \delta_{\widehat\xi_j}}{\sum_{j=1}^N\exp\left(-z^\top \,\psi(\widehat\xi_j)\right)},
\end{equation*}
see \cite[Lemma~2]{ref:sutter-JMLR-19}. By construction, the smoothed dual objective function~$F_\eta$ is $\eta_2$-strongly concave and differentiable. Its gradient can be expressed in terms of the parametric optimizers~$x^\star_z$ and~$\mathbb Q_z^\star$ as
\begin{equation*}
\nabla F_\eta(z) = -x^\star_z +\mb E_{\mb Q_z^\star}[\psi(\xi)] - \eta_2 z = G_\eta(z),
\end{equation*}
where $G_\eta$ is defined in~\eqref{eq:def:gradient:FGA}; see also~\cite[Theorem~1]{nesterov05}. In addition, as shown in \cite[Theorem~1]{nesterov05}, the gradient function~$G_\eta$ is Lipschitz continuous with a Lipschitz constant~$L_\eta$ that satisfies
\begin{align*}
    L_\eta &= 1/\eta_1 + \eta_2 + 
    \left( \sup_{\lambda\in\mathbb{R}^d, \mb Q\in\mc M(\Xi)} \left\{ \lambda^\top \mb E_{\mb Q}[\psi(\xi)]\ : \ \|\lambda\|_2=1, \|\mb Q\|_{\mathsf{TV}}=1 \right\} \right)^2 \\
    &\leq 1/\eta_1 + \eta_2 + \left( \sup_{\lambda\in\mathbb{R}^d, \mb Q\in\mc M(\Xi)} \left\{ \|\lambda\|_2 \|\mb E_{\mb Q}[\psi(\xi)]\|_2 \ : \ \|\lambda\|_2=1, \|\mb Q\|_{\mathsf{TV}}=1 \right\} \right)^2 \\
    &= 1/\eta_1 + \eta_2 +(
    \max_{\xi\in\Xi}\|\psi(\xi)\|_\infty)^2<\infty.
\end{align*}
Therefore, the smoothed dual optimization problem
\begin{equation}\label{eq:regularized:dual}
\sup_{z\in\Re^d} F_\eta(z)
\end{equation}
has a smooth and strongly concave objective function, implying that it can be solved highly efficiently via fast gradient methods. When solving \eqref{eq:regularized:dual} by Algorithm~\hyperlink{algo:1}{1}, we can use its outputs~$z_k$ to construct candidate solutions~$\widehat{\mb Q}_{k}$ for the primal (non-regularized) problem~\eqref{eq:primal:problem} as described in~\eqref{eq:estimates:primal:dual}. These candidate solutions satisfy the optimality and feasibility guarantees~\eqref{eq:thm:primal:optimality:cts} and~\eqref{eq:thm:primal:feasibility:cts}, which can be derived by using the techniques developed in~\cite{ref:devolder-12}. A detailed derivation using our notation is also provided in~\cite[Appendix~A]{ref:sutter-JMLR-19}. We highlight that~\eqref{eq:thm:primal:optimality:cts} and~\eqref{eq:thm:primal:feasibility:cts} critically rely on Assumption~\ref{ass:slater}, which implies via \cite[Lemma 1]{ref:Nedic-08} that the norm of the unique maximizer of the regularized dual problem \eqref{eq:regularized:dual} is bounded above by~$C/\delta$, where $C$ and $\delta$ are defined as in \eqref{eq:definitions:algo:cont}.
\end{proof}

\begin{proof}[Proof of Proposition~\ref{prop:duality:DRO}]
By the definition of the DRO predictor~$R^\star$ in \eqref{eq:def:DRO:predictor}, we have
\begin{align*}
    R^\star(\theta,\mathbb{P}')
    &= \sup_{\mb Q\in\mc P(\Xi)}\left\{ \mb E_{\mb Q}[L(\theta,\xi)] \ : \ \D{\mb P'}{\mb Q}\leq r, ~\mb E_{\mb Q}[\psi(\xi)]\in E \right\} \\
    &= \sup_{\mb Q\in\mc P(\Xi)}\left\{ \mb E_{\mb Q}[L(\theta,\xi)] - \sup_{z\in\mb R^d}\{ z^\top \mb E_{\mb Q}[\psi(\xi)] - \sigma_E(z) \} \ : \ \D{\mb P'}{\mb Q}\leq r \right\} \\
    &=\inf_{z\in\mb R^d} \sup_{\mb Q\in\mc P(\Xi)} \left\{ \mb E_{\mb Q}[L(\theta,\xi)-z^\top \psi(\xi)] + \sigma_E(z) \ : \ \D{\mb P'}{\mb Q}\leq r \right\} \\
    &= \displaystyle \inf_{z\in\mathbb{R}^d}\left\{ \begin{array}{cl}
   \inf\limits_{\alpha\in\mathbb{R}}  &  \alpha + \sigma_E(z) - e^{-r}\exp\left( \mathbb{E}_{\mathbb{P}'}[\log (\alpha - L(\theta,\xi)+z^\top \psi(\xi))]\right) \\
     \st & \alpha \geq \max_{\xi\in\Xi} L(\theta,\xi) - z^\top \psi(\xi) 
\end{array}\right.
\end{align*}
where the second equality holds because the convex conjugate of the support function~$\sigma_E$ is the indicator function~$\delta_E:\Re^{d}\to [0,\infty]$ defined through~$\delta_E(x)=0$ if~$x\in E$ and $\delta_E(x)=\infty$ if~$x\notin E$,
and the third equality follows from Sion's minimax theorem, which applies because the relative entropy $\D{\bar{\mathbb{Q}}}{\mathbb{P}}$ is convex in~$\bar{\mathbb{Q}}$, while the distribution family~$\mathcal{P}(\Xi)$ is convex and weakly compact. Finally, the fourth equality follows from \cite[Proposition~5]{ref:vanParys:fromdata-17},  which applies because~$r>0$ and because the modified loss function~$L(\theta,\xi)-z^\top \psi(\xi)$ is continuous in~$\xi$ for any fixed~$\theta$ and~$z$. The last expression ins manifestly equivalent to~\eqref{eq:DRO:predictor:duality:formula}, and thus the claim follows.
\end{proof}

The following corollary of Proposition~\ref{prop:duality:DRO} establishes that the DRO predictor~$R^\star$ is continuous. This result is relevant for Theorem~\ref{thm:admissibility}.


\begin{corollary}[Continuity of~$R^\star$]\label{corollary:dual:multiplier:bound}
If $r>0$, $0\in\interior(E)$ and for every~$z\in\mathbb R^d$ there exists~$\xi\in\Xi$ such that~$z^\top \psi(\xi)>0$, then the DRO predictor~$R^\star$ is continuous on~$\Theta\times\Pi$.

\end{corollary}
\begin{proof}
Since $r>0$, we may use Proposition~\ref{prop:duality:DRO} to express the DRO predictor as
\begin{equation}\label{eq:DRO:over:g}
R^\star(\theta,\mb P') = \inf_{z\in\mb R^d}\varphi_E(\theta,z,\mb P')
\end{equation}
for all~$\theta\in\Theta$ and~$\mathbb P'\in\Pi$, where the parametric objective function $\varphi_E$ is defined through
\begin{align*}
  \varphi_E(\theta,z,\mb P') =  \inf\limits_{\alpha\ge \underline \alpha(\theta, z)}   \alpha + \sigma_E(z) - e^{-r}\exp\left( \mathbb{E}_{\mathbb{P}'}[\log (\alpha - L(\theta,\xi)+z^\top \psi(\xi))]\right)
\end{align*}
with~$\underline \alpha(\theta, z) =\max_{\xi\in\Xi} L(\theta,\xi) - z^\top \psi(\xi)$. Note that the support function~$\sigma_E$ is continuous because~$E$ is compact. Applying \cite[Proposition~6]{ref:vanParys:fromdata-17} to the modified loss function $L(\theta,\xi)-z^\top \psi(\xi)$ thus implies that~$\varphi_E$ is continuous on~$\Theta\times\mathbb R^d\times \Pi$. To bound~$\varphi_E$ from below by a coercive function, we define
\[
    \kappa = \min_{\|z\|_2=1}\min_{\mb Q\in\Pi} ~\sigma_E(z) - e^{-r}z^\top \mb E_{\mb Q}[\psi(\xi)],
\]
which is a finite constant. Indeed, $\sigma_E$ is continuous because~$E$ is compact, and~$\mb E_{\mb P'}[\psi(\xi)]$ is weakly continuous in~$\mb P'$ because~$\psi$ is a continuous and bounded function on the compact set~$\Xi$. In addition, the unit sphere in~$\mathbb R^d$ is compact, and the set~$\Pi$ is weakly compact. Therefore, both minima in the definition of~$\kappa$ are attained at some~$z^\star\in\Re^d$ with~$\|z^\star\|_2=1$ and some~$\mathbb Q^\star\in\Pi$, respectively. As~$0\in\interior(E)$ and~$z^\star\neq 0$, we have $\sigma_E(z^\star)>0$. In addition, as $\mb Q^\star\in\Pi$, we have~$\mb E_{\mb Q^\star}[\psi(\xi)]\in E$, which implies that $(z^\star)^\top \mb E_{\mb Q^\star}[\psi(\xi)] \leq \sigma_E(z^\star)$. Again as $r>0$, this reasoning ensures that
\begin{equation*}
    \kappa = \sigma_E(z^\star) - e^{-r} (z^\star)^\top \mb E_{\mb Q^\star}[\psi(\xi)] >0.
\end{equation*}
Similarly, we introduce the finite constant
\[
    \underline L = \min_{\theta\in\Theta} \min_{z\in\Re^d} \min_{\xi\in\Xi}  ~ (1-e^{-r}) \underline \alpha(\theta, z)+ e^{-r} L(\theta,\xi).
\]
To see that~$\underline L$ is bounded below, note that the definition of~$\underline\alpha$ and the subadditivity of the minimum operator lead to the estimate
\begin{align*}
    \underline L &\geq (1-e^{-r})\min_{\theta\in\Theta} \min_{\xi\in\Xi}  ~  L(\theta, \xi) + (1-e^{-r}) \min_{z\in\Re^d}\max_{\xi\in\Xi} ~(-z)^\top\psi(\xi) +  e^{-r}\min_{\theta\in\Theta} \min_{\xi\in\Xi} ~ L(\theta,\xi)\\
    &= \min_{\theta\in\Theta} \min_{\xi\in\Xi}  ~ L(\theta, \xi) + (1-e^{-r})\min_{z\in\Re^d}\max_{\xi\in\Xi} ~(-z)^\top\psi(\xi).
\end{align*}
The first term in the resulting lower bound is finite because~$L$ is continuous, while~$\Theta$ and~$\Xi$ are compact. The second term is also finite because the convex function~$\max_{\xi\in\Xi} ~(-z)^\top\psi(\xi)$ is continuous in~$z$ thanks to the continuity of~$\psi$ and the compactness of~$\Xi$. In addition, $\max_{\xi\in\Xi} ~(-z)^\top\psi(\xi)$ is also coercive in~$z$ because of the assumption that for every~$z\in\mathbb R^d$ there exists~$\xi\in\Xi$ with~$z^\top \psi(\xi)>0$. 

The above preparatory arguments imply that
\begin{align*}
    & \varphi_E(\theta,z,\mb P')
    \geq \inf_{\alpha\ge \underline \alpha(\theta, z)} (1-e^{-r}) \alpha + \sigma_E(z) +e^{-r}\mathbb{E}_{\mathbb{P}'}[L(\theta,\xi)]-e^{-r}z^\top \mathbb{E}_{\mathbb{P}'}[\psi(\xi)] \\
    & = (1-e^{-r}) \underline \alpha(\theta, z)+ e^{-r}\mathbb{E}_{\mathbb{P}'}[L(\theta,\xi)] + \left(\sigma_E\left(\frac{z}{\|z\|_2}\right) - e^{-r} \left(\frac{z}{\|z\|_2}\right)^\top\mathbb{E}_{\mathbb{P}'}[ \psi(\xi)]\right) \|z\|_2 \\
    &\geq \underline{L} + \kappa \|z\|_2,
\end{align*}
where the first inequality exploits Jensen's inequality, the equality holds thanks to the positive homogeneity of the support function~$\sigma_E$ and the trivial observation that $e^{-r}<1$, and the second inequality follows from the definitions of~$\underline L$ and~$\kappa$ and the assumption that~$\mb P'\in\Pi$. We thus have
\begin{subequations}
\begin{align}
    \label{eq:lb-varphi}
    \varphi_E(\theta,z,\mb P')\geq \underline L+\kappa \|z\|_2\quad \forall \theta\in\Theta, \;\forall z\in\mathbb R^d,\;\forall \mathbb P'\in\Pi.
\end{align}
Next, define
\[
    \overline L = \max_{\theta\in\Theta} \max_{\xi\in\Xi} L(\theta,\xi),
\]
and note that
\begin{align}
    \label{eq:ub-varphi}
    \inf_{z\in\mb R^d}\varphi_E(\theta,z,\mb P')= R^\star(\theta,\mb P^\star) \leq \overline L\quad \forall \theta\in\Theta, \;\forall \mathbb P'\in\Pi.
\end{align}
\end{subequations}
Taken together, the estimates~\eqref{eq:lb-varphi} and~\eqref{eq:ub-varphi} imply that
\[
    R^\star(\theta,\mb P') = \inf_{z\in\mb R^d} \left\{ \varphi_E(\theta,z,\mb P') : \|z\|_2\leq \frac{\overline L-\underline L}{\kappa}\right\},
\]
which in turn implies via Berge's maximum theorem \cite[pp.~115--116]{berge1997topological} and the continuity of the objective function~$\varphi_E$ on~$\Theta\times\Re^d\times\Pi$ that the DRO predictor~$R^\star$ is indeed continuous on~$\Theta\times\Pi$.

\end{proof}

\subsection{Auxiliary results for Section~\ref{sec:numerical:experiments}}\label{app:numerics}

\textbf{Classification under covariate shift.} 
We construct a synthetic training data consisting of feature vectors~$\widehat x_i$ and corresponding labels~$\widehat y_i$. Under the training distribution~$\mb P$, the feature vectors are uniformly distributed on $[0,1]^{m-1}$, where $m\geq 2$, and the labels are set to
$\widehat y_i = 1$ if $\frac{1}{m-1}\sum_{j=1}^{m-1} (\widehat{x}_i)_j > \frac{1}{2}$ and $\widehat y_i = -1$ otherwise. By construction, we thus have~$\mb E_{\mathbb{P}}[(x,y)]=(0,0)$. The test distribution~$\mathbb{P}^\star$ differs from~$\mb P$. Specifically the probability density function of the features under~$\mb P^\star$ is set to
\begin{equation*} 
    \textstyle{p^\star(x) = \frac{2}{m-1} \sum_{j=1}^{m-1} x_j \quad \forall x\in[0,1]^{m-1},}
\end{equation*}
while the conditional distribution of the labels given the features is the same under~$\mb P$ and~$\mb P^\star$. A direct calculation then reveals that $\mb E_{\mathbb{P}^\star}[x_j]=\frac{m-2}{2(m-1)}+\frac{2}{3(m-1)}=\mu^\star>0$ for all $j=1,\hdots, m-1$. Similarly, one can show that~$\mb E_{\mathbb{Q}}[y]>0$. In the numerical experiments we assume that both~$\mb P$ and~$\mb P^\star$ are unknown. However, we assume to have access to~$N$ i.i.d.\ samples from~$\mb P$, and we assume that~$\mb P^\star$ is known to satisfy~$\mb E_{\mb P^\star}[\psi(\xi)]\in E$, where~$\psi(x,y) = (x,y)$ and~$E=[(\mu^\star-\varepsilon)\cdot 1, (\mu^\star-\varepsilon)\cdot 1]$ for some~$\varepsilon>0$ that is sufficiently small to ensure that~$0\notin E$. This implies that~$\mb P\notin\Pi$.

\begin{figure}[t!p]
\begin{center}
   \includegraphics[width=0.85\textwidth]{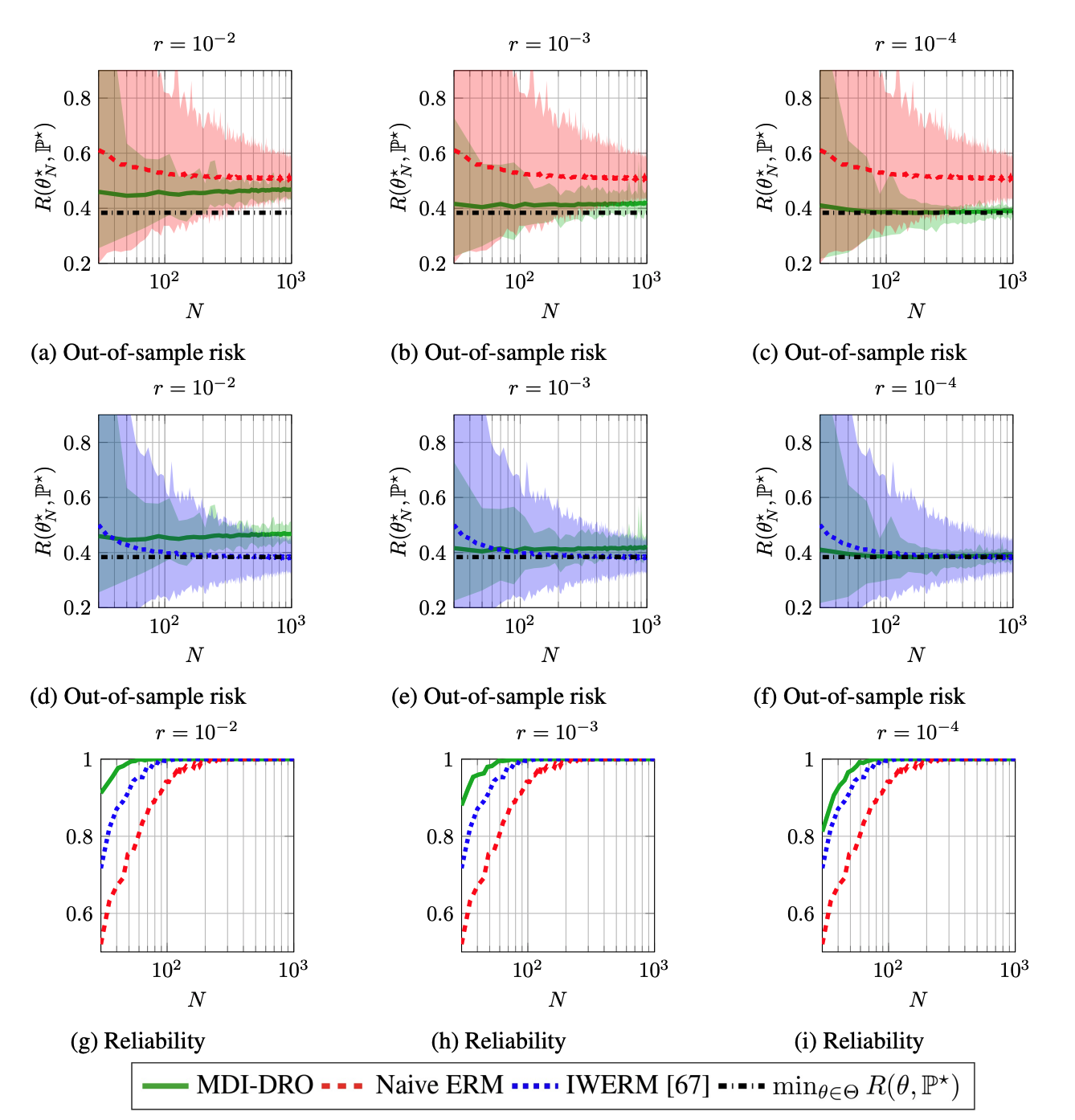}
   \end{center}
         \caption{Additional results for the synthetic dataset with $m=6$ (see also Figure~\ref{fig:toy:out-of-sample:high:D}). Shaded areas and lines represent ranges and mean values across 1000 independent experiments, respectively.} \label{fig:classification:appendix:synthetic}
\end{figure}

\textbf{Inventory control model.}
Consider an inventory that stores a homogeneous good, and let the state variable~$s_i$ represent the stock level at the beginning of period~$i$. The control action~$a_i$ reflects the order quantity in period~$i$, and we assume that any orders are delivered immediately at the beginning of the respective periods. The disturbance~$\zeta_i$ represents an uncertain demand revealed in period~$i$. We assume that the demands are i.i.d.\ across periods and follow a geometric distribution on~$\mb N\cup\{0\}$ with success probability~$\lambda\in(0,1)$. The inventory capacity is denoted by~$\gamma\in\mathbb{N}$, and any orders that cannot be stored are lost. Similarly, we assume that any demand that cannot be satisfied is also lost. The system equation describing the dynamics of the stock level is thus given by
\begin{equation*}
s_{i+1} =  \max\{0,\min\{\gamma, s_i + a_i\} - \zeta_i\}\quad \forall i=0,1,2, \hdots,
\end{equation*}
see also \cite{ref:Hernandez-96}. Our aim is to estimate the long-run average cost generated by a prescribed ordering policy, assuming that the (uncertain) cost incurred in period~$i\in\mb N$ can be expressed as
\begin{equation*}
r(s_i, a_i, \zeta_i)= p a_i + h(s_i + a_i) - v \min\{s_i + a_i,\zeta_i\}.
\end{equation*}
The three terms in the above expression capture the order cost, the inventory holding cost and the profit from sales, where~$p>0$ and~$h>0$ denote the costs for ordering or storing one unit of the good, while~$v>0$ denotes the unit sales price. The expected per period cost thus amounts to
\begin{equation*} 
\textstyle c(s_i,a_i) = pa_i + h(s_i+a_i) - v \frac{(1-\lambda)}{\lambda}\left(1-(1-\lambda)^{(a_i+s_i)}\right).
\end{equation*}
The simulation results shown in Figure~\ref{fig:OPE:inventory} are based on an instance of the inventory control model with state space~$\mathcal{S}=\{1,2,\hdots,5\}$, action space~$\mathcal{A}=\{1,2,\hdots,4\}$, and parameters~$\lambda = 0.2$, $\gamma=5$, $p=0.6$, $h=0.3$ and~$v=1$. The threshold for computing the modified IPS estimator is set to~$\beta=4$. It is easy to verify that, under this model parameterization, the cost function~$c(s_i,a_i)$ is invertible in the sense that~$s_i$ and~$a_i$ are uniquely determined by~$c(s_i,a_i)$; see also Example~\ref{ex:OPE:part:1}.

\bibliographystyle{plain}
\bibliography{references}
\end{document}